\documentclass{article}

\usepackage[dvipsnames,table,xcdraw]{xcolor}

\usepackage{titletoc}
\usepackage[page,header]{appendix}

\usepackage{microtype}
\usepackage{graphicx}
\usepackage{subfigure}
\usepackage{booktabs} 
\usepackage{multirow}
\usepackage{colortbl}
\usepackage{placeins}

\usepackage[accepted]{icml2025}

\usepackage{bm}
\usepackage{nicefrac}
\usepackage{amsmath}
\usepackage{amssymb}
\usepackage{mathtools}
\usepackage{amsthm}
\usepackage{soul}

\definecolor{ballblue}{HTML}{338EA7}
\definecolor{lightseagreen}{HTML}{759D39}
\definecolor{lightred}{HTML}{DD7769}
\definecolor{org}{HTML}{F8A145}
\definecolor{blu}{HTML}{63ACE5}
\definecolor{c1}{HTML}{41B3A3}
\definecolor{c2}{HTML}{3500D3}
\definecolor{test}{HTML}{E5B8FD}
\definecolor{bdy}{HTML}{FDADAD}
\definecolor{corr}{HTML}{B2E1EA}
\definecolor{wrng}{HTML}{CFE8BD}

\usepackage[colorlinks,linkcolor=blue,citecolor=ForestGreen,urlcolor=org, linktoc=all]{hyperref}

\usepackage[capitalize,noabbrev]{cleveref}
\usepackage{algorithm}
\usepackage{algorithmic}
\usepackage{bm}
\usepackage{enumitem}

\theoremstyle{plain}
\newtheorem{theorem}{Theorem}[section]

\newtheorem{lemma}[theorem]{Lemma}

\theoremstyle{definition}
\newtheorem{definition}[theorem]{Definition}

\theoremstyle{remark}
\newtheorem{remark}[theorem]{Remark}

\usepackage[textsize=tiny]{todonotes}

\usepackage{alltt,capt-of}
\usepackage{tcolorbox}
\newenvironment{qbox}
{\begin{tcolorbox}[colback=white, width=0.95\linewidth, center, left=2pt,right=2pt,top=1pt,bottom=1pt]}
{\end{tcolorbox}}

\icmltitlerunning{Beyond Progress Measures: Theoretical Insights into the Mechanism of Grokking}

\begin{document}

\twocolumn[
\icmltitle{Beyond Progress Measures:\\Theoretical Insights into the Mechanism of Grokking}



\icmlsetsymbol{equal}{*}

\begin{icmlauthorlist}
\icmlauthor{Zihan Gu}{equal,iie,ucas}
\icmlauthor{Ruoyu Chen}{equal,iie,ucas}
\icmlauthor{Hua Zhang}{iie,ucas}
\icmlauthor{Yue Hu}{iie,ucas}
\icmlauthor{Xiaochun Cao}{sysu}
\end{icmlauthorlist}

\icmlaffiliation{iie}{Institute of Information Engineering, Chinese Academy of Sciences, Beijing 100093, China}
\icmlaffiliation{ucas}{School of Cyber Security, University of Chinese Academy of Sciences, Beijing 100049, China}
\icmlaffiliation{sysu}{School of Cyber Science and Technology, Shenzhen Campus of Sun Yat-sen University, Shenzhen 518107, China}



\vskip 0.3in
]



\printAffiliationsAndNotice{\icmlEqualContribution} 

\begin{abstract}
Grokking, referring to the abrupt improvement in test accuracy after extended overfitting, offers valuable insights into the mechanisms of model generalization. Existing researches based on \textit{progress measures} imply that grokking relies on understanding the optimization dynamics when the loss function is dominated solely by the weight decay term. However, we find that this optimization merely leads to token uniformity, which is not a sufficient condition for grokking. In this work, we investigate the grokking mechanism underlying the Transformer in the task of prime number operations. Based on theoretical analysis and experimental validation, we present the following insights: (i) \textit{The weight decay term} encourages uniformity across all tokens in the embedding space when it is minimized. (ii) \textit{The occurrence of grokking} is jointly determined by the uniformity of the embedding space and the distribution of the training dataset. Building on these insights, we provide a unified perspective for understanding various previously proposed progress measures and introduce a novel, concise, and effective progress measure that could trace the changes in test loss more accurately. Finally, to demonstrate the versatility of our theoretical framework, we design a dedicated dataset to validate our theory on ResNet-18, successfully showcasing the occurrence of grokking. The code is released at \url{https://github.com/Qihuai27/Grokking-Insight}.
\end{abstract}

\section{Introduction}
\label{intro}

\textsc{Grokking}, or delayed generalization, refers to the phenomenon where a model’s test accuracy abruptly improves after a prolonged period of overfitting. This phenomenon was first identified by Power \textit{et al.}~\cite{power2022grokking} in the context of operations within a prime number field. Investigating the occurrence of the grokking phenomenon offers valuable insights into the underlying mechanisms of representation learning and generalization in neural networks~\cite{humayun2024deep,chen2024less}.


Previous studies~\cite{power2022grokking,liu2022towards,nandaProgressMeasuresGrokking2022,humayun2024deep} have predominantly focused on predicting or monitoring the occurrence of grokking, emphasizing the importance of these tasks as a means to uncover its underlying causes.
Such efforts are often framed in terms of \textit{progress measures}~\cite{nandaProgressMeasuresGrokking2022,clauw2024information,humayun2024deep,furutaInterpretingGrokkedTransformers2024}, which aim to identify specific activation values within the model to characterize the phenomenon during parameter updates.
However, a critical theoretical gap persists: existing methods fall short of adequately explaining why a well-designed progress measure can effectively capture the grokking process.
Moreover, the insights gained from these studies often provide only a partial understanding, falling short of offering a comprehensive explanation of the underlying mechanisms~\cite{zhuCriticalDataSize2024,huangUnifiedViewGrokking2024,leeGrokfastAcceleratedGrokking2024,parkAccelerationGrokkingLearning2024}.
To address these limitations, we aim to investigate the following question:
\begin{qbox}
    \begin{center}
    \footnotesize \textbf{\textit{What are the underlying mechanisms driving grokking?}}
    \end{center}
\end{qbox}

\begin{figure*}[t]
    \centering
    \vspace{-8pt}
    \includegraphics[width=\textwidth]{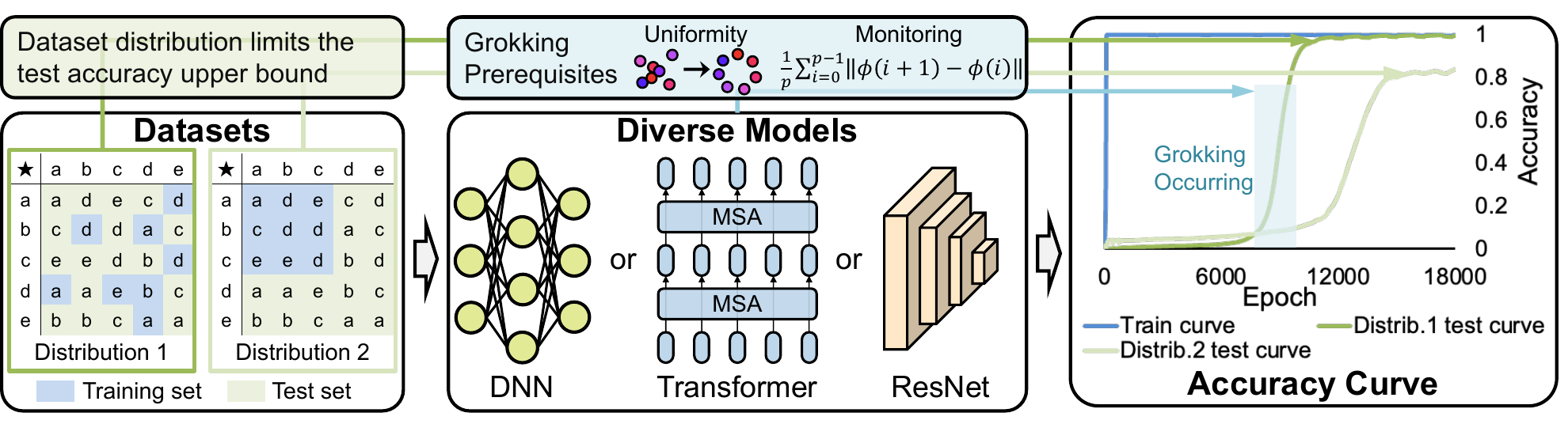}\vspace{-10pt}
    \caption{Grokking occurs from two factors: the \textbf{uniformity} of embeddings under weight decay and the \textbf{dataset distribution}. Embedding uniformity arises from parameter updates guided by minimizing the weight decay term, enabling us to track grokking by monitoring this uniformity. The distribution of the training set determines the upper limit of test accuracy. The above insights are applicable to diverse models, including DNN, Transformer and ResNet.}
    \label{abstract}
    \vspace{-18pt}
\end{figure*}

To answer this question, we investigate the operations within a prime-number field task, as proposed by Power \textit{et al.}~\cite{power2022grokking}, utilizing the single-layer transformer-based decoder~\cite{vaswani2017attention}. The analysis can be divided into two stages: (1) What are the properties of the convergence domain of the model parameters when the loss function includes only weight decay terms? (2) Why does the test accuracy reach 100$\%$ in this convergence domain? A representative kernel-based study considers the former problem by analyzing the effect of the weight decay term in a two-layer MLP \cite{mohamadiWhyYouGrok2024}. However, the proof hinges on the piecewise-linearity of MLPs and the specific properties of the addition task and therefore does not account for grokking when other operations are performed on the prime field. Moreover, there is currently no related research addressing the second question.
In the first stage, we demonstrate that the weight decay term significantly impacts the parameter convergence domain, resulting in uniformity in the embedded values. We further prove that this \ul{\textit{uniformity is the key factor in improving test accuracy}}. Additionally, we establish the necessary and sufficient conditions for grokking to occur in such tasks and design experiments to validate these findings. Building on this, we analyze and demonstrate that the \ul{\textit{the training set distribution defines the upper bound for test accuracy}}. While previous studies often report 100\% test accuracy due to random training set proportions, we show that by deliberately selecting the distribution of the training set, the upper limit of test accuracy can be effectively controlled. Figure~\ref{abstract} provides a conceptual illustration of the factors contributing to grokking.

Expanding on the above theoretical analysis, we observe that previous research on progress measures has primarily focused on approximating the uniformity of embedded values~\cite{liu2022towards,millerMeasuringSharpnessGrokking2024,wangGrokkedTransformersAre2024}. However, methods relying on Fourier decomposition~\cite{nandaProgressMeasuresGrokking2022} or local complexity~\cite{humayun2024deep} are often overly complex and impractical. In contrast, we propose a concise and straightforward progress measure grounded in our theoretical insights. This measure directly monitors the trends in the loss function on the test set, providing a more efficient and practical alternative.

Furthermore, we demonstrate that any network whose feature operations become quasi-linear after capturing the main features will inevitably exhibit grokking on certain tasks. This implies that deeper networks may also be capable of exhibiting grokking. However, existing reports of this phenomenon are largely confined to relatively small networks, such as shallow DNNs or smaller transformers~\cite{power2022grokking,liu2022towards,gromov2023grokkingmodulararithmetic,Lyu2024dichotomy}. To show that grokking can indeed occur in deeper architectures, we constructed a novel task based on our theoretical findings and successfully demonstrated grokking on ResNet-18~\cite{he2016deep}, further validating the robustness and completeness of our theoretical framework.


Finally, we explore whether this research can enhance our understanding and improvement of network architectures. Specifically, we investigate how the transformer model encodes and processes knowledge from the dataset, defining three types of separability: topological separability, convex separability, and projected convex separability. The grokking phenomenon observed in operations over prime fields indicates that the transformer tends to encode the knowledge of the training dataset in a projected convex separable form. Interestingly, the grokking phenomenon diminishes as the number of transformer layers increases. This suggests that the convex hull representing the stored knowledge becomes smaller as the architecture deepens, leading us to speculate that deeper transformer architectures possess greater knowledge capacity.


To summarize, our contributions are as follows:
\begin{itemize}[itemsep=2pt,topsep=0pt,parsep=0pt]
    \item We present a comprehensive mathematical proof of the underlying mechanism behind grokking, demonstrating that the uniformity of embedded values drives grokking, while the distribution of training data determines the upper bound of test accuracy.
    \item We provide a systematic perspective on understanding progress measures associated with grokking and propose a novel, concise, and efficient approach for monitoring trends in the test set's loss function.
    \item We construct a novel task and successfully demonstrate the grokking phenomenon on the deep network ResNet-18, further validating the completeness of our theoretical framework.
    \item We propose a new conjecture on how the transformer architecture processes knowledge from the training set and offer a structured representation of this mechanism.
\end{itemize}


\section{Related Work}
\label{related}

\textbf{Grokking:} Grokking was first proposed in addition over a prime field~\cite{power2022grokking}, which happened on a model of a two-layers transformer decoder. Initial understanding of grokking focused on the size of the dataset and some studies proposed the concept of 'critical dataset size'~\cite{zhuCriticalDataSize2024,huangUnifiedViewGrokking2024}. The essence of these methods is to explain mutation behavior through the linear variation of dataset size. Some studies have also recognized that grokking might be a common phenomenon in classification tasks, prompting researchers to approach the problem from a structural perspective~\cite{liu2022towards,thilakSlingshotMechanismEmpirical2022a}. However, the number of classification tasks in which the grokking phenomenon has been observed is limited, leading some to speculate that grokking is a result of the Transformer architecture~\cite{wangGrokkedTransformersAre2024}. This view was quickly overturned by the grokking phenomenon that appeared on DNNs~\cite{gromov2023grokkingmodulararithmetic,Lyu2024dichotomy,mallinarEmergenceNonneuralModels2024}. Some studies also focus on changing the architecture to investigate variations in the phenomenon~\cite{parkAccelerationGrokkingLearning2024,kuninGetRichQuick2024,leeGrokfastAcceleratedGrokking2024}. And there have always been researchers who associate grokking with emergence, and notable work in this area~\cite{mallinarEmergenceNonneuralModels2024,heLearningGrokEmergence2024,zhaoUncoveringHowLarge2024}.

\textbf{Progress Measure:} The concept of \textit{progress measures} was first introduced in research on sparse parities~\cite{barakHiddenProgressDeep2022}, essentially as a smooth function of continuous changes in activation values used to predict model behavior. The first perspective is based on frequency and Fourier coefficients, inspired by circuit signal analysis~\cite{nandaProgressMeasuresGrokking2022,zhouRationaleFrequencyPerspective2024,furutaInterpretingGrokkedTransformers2024}. The second perspective is local complexity, derived from linear region analysis~\cite{humayun2024deep}. The third perspective is based on information theory~\cite{clauw2024information}. Recently, some methods have aimed to provide a more unified 
perspective~\cite{yunisApproachingDeepLearning2024,songUnveilingDynamicsInformation2024a}.

\section{Why Grokking Happens: A Complete Mathematical Analysis}\label{math_analysis}

We begin by introducing an operator-based formulation of the problem  in Section~\ref{pre}, which we then decompose into two distinct steps. In Section~\ref{pro}, we analyze how the weight decay term in the loss function guides the model parameters to converge after the training accuracy reaches 100\%. In Section~\ref{acc}, we explore why this specific convergence direction results in improved test accuracy. Detailed proofs of the theorems are provided in the Appendix.

\subsection{Preliminaries}
\label{pre}

Our research is based on operation tasks on the prime number field $\mathbb{Z}_p$, focusing on the problem of outputting $f(i,j) \; mod \; p$, where $f$ is a two-variable polynomial function and $i,j \in \mathbb{Z}_p$. This task is performed by a single-layer transformer on inputs of the form $(i, j, \mathsf{cls})$. Unless otherwise specified, in the following text we consider $f(i,j)=i+j$ and denote the remainder of $f(i,j)$ modulo $p$ by $r(i,j)$.


For simplicity, we adopt a formal operator framework and use geometric terminology to characterize the model’s operational process. A single-layer transformer decoder comprises the following operators: the identity operator \(I_d\) (corresponding to the residual connection), the normalization operator \(n\) (corresponding to attention score normalization), the truncation operator \(r\) (corresponding to the MLP activation function), linear operators \(l_O, l_V, l_1, l_2, l\), and the bilinear operator \(\beta_{QK}\) (corresponding to attention weight calculation). Consequently, the transformer's two primary component expressions can be written as follows:
\begin{equation}
     \mathsf{SelfAtt}(X)=I_d(X) + l_O \circ (n\circ \beta_{QK}(X,X)\cdot l_V(X)),
\end{equation}
\begin{equation}
    \mathsf{MLP}(X) = I_d(X) + l_1 \circ r \circ l_2(X).
\end{equation}

We consider the dataset $I \times J$ and $I = J = \mathbb{Z}_p$ with embedding function $\phi:\mathbb{Z}_p \to \mathbb{R}^n$. So the token input into the transformer has the form of $(\phi(i),\phi(j),\phi(\mathsf{cls}))$ where $i \in I$ and $j \in J$. We take $I' \times J'$ as the training set and $I'' \times J''$ as the test set. Now we define an equivalence relation in the dataset.

\begin{definition}
   Define the relation \(\sim\) on pairs \((i, j)\) with \(i \in I\) and \(j \in J\) by
   \begin{equation*}
    (i_1, j_1) \sim (i_2, j_2)
   \quad\Longleftrightarrow\quad
   r(i_1, j_1) \;=\; r(i_2, j_2).   
   \end{equation*}  
\end{definition}

Each element in the quotient $I \times J / \sim$ represents an equivalence class of $I \times J $  with respect to $\sim$. We call the model \textbf{masters} $I \times J / \sim$ if it can output correct $r(i,j)$ on the whole set $I \times J$. The learning goal of our model is to master $I'' \times J'' / \sim$ through learning $I'\times J'/\sim$. Next, two preliminary lemmas are introduced as a foundation for the proof. Lemma~\ref{lem:softmax} provides an alternative representation of the optimization objective, while Lemma~\ref{lem:norm} shows that changes in the operator norm are directly related to changes in the weight decay term.

\begin{lemma}
\label{lem:softmax}
    Let $(\bar{\phi(i)},\bar{\phi(j)},\bar{\phi(\mathsf{cls})} )= \mathsf{MLP}  \circ \mathsf{SelfAtt} (\phi(i),\phi(j),\phi(\mathsf{cls}))$ and $\bar{c_{i,j}} = l(\bar{\phi(\mathsf{cls})})$. If the model can give the correct output, there exists a set of vectors $\{c_i\},i=0,1,...,p-1$ and the relevant small positive number $\epsilon_i$ such that $\|\bar{c_{i,j}}-c_{r(i,j)} \|< \epsilon_{r(i,j)}$.
\end{lemma}

\begin{remark}
    As the last layer of the classification head, the softmax function actually takes the label of the item with the largest weight as the class, which actually corresponds to a cone on $\mathbb{R}^n$. We can take a vector in the cone and use the distance from the model output to this vector to measure whether the optimization is complete.
\end{remark}

\begin{lemma}
\label{lem:norm}
    For linear operator $l:\mathbb{R}^n \to \mathbb{R}^m$, if it has a matrix expression $M$ under two sets of basis vectors, then the operator norm of l and the Frobenius norm of $M$ has the following relationship:
    \begin{equation}
        \|l\| \leq \|M\|_F \leq \sqrt{\min\{m,n\}}  \|l\|.
    \end{equation}
\end{lemma}

The Frobenius norm of a matrix is defined as the square root of the sum of the squares of all its entries. By Lemma~\ref{lem:norm}, the variation in the weight decay term can be estimated through changes in the sum of the norms of the involved operators.

\subsection{Properties of the Region with Minimal Weight}
\label{pro}

We use $W$ to represent the weight decay term in the loss function and write the set of all model weights that make $W < E$ as $\Omega$. We call $\Omega$ the region with minimal weight. To understand the properties of $\Omega$, we first need to write the final output feature formula: $\bar{\phi(\mathsf{cls})}=I_d(\phi(\mathsf{cls})) + l_1 \circ r \circ l_2(\phi(\mathsf{cls})) + (I_d + l_1 \circ r \circ l_2)\circ l_O \circ (n \circ (\sum_{t\in\{i,j,\mathsf{cls}\}}\beta_{QK}(\phi(t),\phi(\mathsf{cls}))l_V(\phi(t)))$. Although this expression appears highly complex, it is the linear portion within the summation that truly matters. Substituting this formula into Lemma~\ref{lem:softmax} yields Theorem~\ref{thm:uniemb}.

\begin{theorem}
\label{thm:uniemb}
    For $A \subset I$, $B\subset J$, if the model has mastered the $A \times B / \sim$, then there exist vectors $\{k_i\},i=0,1,2,...,p-1$ and relevant positive numbers $\delta_i$ that for all $(i,j) \in A \times B$,
    \begin{equation}
        \| \phi(i) + \phi(j) - k_{r(i,j)} \| < \delta_{r(i,j)}.
    \end{equation}
\end{theorem}


\begin{remark}
This theorem is given in Liu's paper as $\| \phi(i)+\phi(j)-\phi(i')-\phi(j')\|<\delta \iff r(i,j)=r(i',j')$.
\end{remark}

\begin{remark}
In Theorem~\ref{thm:uniemb}, $k_i$ is not unique, $\delta_i$ is a quantity that is highly related to $k_i$, model architecture, and training set ratio.
\end{remark}

Considering $\| \phi(i) + \phi(j) - k_{r(i,j)} \| < \delta_{r(i,j)}$ and $\| \phi(i+1) + \phi(j) - k_{r(i+1,j)} \| < \delta_{r(i+1,j)}$ for every $i$ in Theorem~\ref{thm:uniemb}, we can obtain Theorem~\ref{thm:med} by applying the triangle inequality to the equations and this theorem will give a direct expression of the progress measure we designed.

\begin{theorem}
\label{thm:med}
    Let $\phi(p)=\phi(0)$, when $W < E$, there exists a positive number $\delta$ such that  $\frac{1}{p} \sum_{i=0}^{p-1} \| \phi(i+1)-\phi(i)\|<\delta $. 
\end{theorem}

\begin{remark}\label{importantremark}
    What we aim to prove is a stronger result: for any element in $I \times J/ \sim$, embeddings uniformly distributed on the $n-1$ dimensional manifold will reduce the sum of norms in Lemma~\ref{lem:norm}, thereby decreasing the weight decay function. This is what we refer to by $\Omega$ representing the “uniformity” of the embedding. Theorem~\ref{thm:med} directly follows from this result.
\end{remark}

\subsection{Improvement of the Test Set Accuracy}
\label{acc}

In this section, we address why the test accuracy increases rapidly with the model parameters in $\Omega$, as described in the following Theorem~\ref{thm:dist}. The idea described in this theorem is simple: if the model wants to accurately output results in the test set, it must "imitate" the results in the training set, and in this task, the imitated object must be nearby.

\begin{theorem}
\label{thm:dist}
    For $i \in I''$ and $j \in J''$, if model's output is equal to $r(i,j)$, one of the following conditions must be true:
    
    (i) There exists $i_0 \in I$ and $j_0\in J$ with $(i_0,j_0)\sim (i,j)$ such that $d[(i,j),(i_0,j_0)] < D_1$;
    
    (ii) There exists $i_0,i_1,i_2,...,i_s \in I$ and $j_0,j_1,j_2,...j_s \in J$ with $(i_0,j_0) \sim (i,j),(i_t,j_t) \sim (i_{s-t},j_{s-t})$ such that $d[(i,j),(i_1,j_1)]+d[(i_s,j_s),(i_0,j_0)]+\sum_{t=1}^{s-1}d[(i_t,j_t),(i_{t+1},j_{t+1})] < D_2$.

In fact, Condition (ii) includes Condition (i), but in actual use, the probability of condition (ii) occurring is extremely small.
\end{theorem}

\begin{remark}
    This theorem may seem complicated, but it actually describes a very simple thing: for this task, the upper limit of the test accuracy that grokking can achieve is determined by the distribution of the training set in the entire data set. This upper limit is not necessarily 1. 
\end{remark}

A necessary condition for the upper limit of the test accuracy that grokking can achieve to be 1 is that the training set must be able to cover the entire test set within a certain Manhattan distance. We will show this result in detail in the experimental part. We can use a simplified structure to discuss this question: given the train data ratio $\alpha$ and an integer $n$ (In practical application n is a quantity linked to Theorem~\ref{thm:med} and Theorem~\ref{thm:dist}.) randomly take a training set $I' \times J'$, let the proportion of points in $I'' \times J''$ whose Manhattan distance to $I' \times J'$ is less than n be $\text{prop}$, when will the expectation of prop reach 1? We can use the upper bound of the joint probability to estimate this value. We give the estimation formula and the specific proof can be found in the appendix.

\begin{theorem}
\label{thm:simpledis}
    If $E[\text{prop}]=1$($E[\cdot]$ represents \textit{Mathematical Expectation}),  let C represent the number of points whose Manhattan distance is less than n for a single point, we will have $\alpha \geq \frac{2lnp}{C} $.
\end{theorem}

In fact, this lower bound is also the proportion of training sets required to stably observe grokking under the assumption of randomly taking training sets.

\section{A Unified Perspective of Progress Measures}
\label{progress measure}

In this section, we will give a detailed description of \textit{progress measures}. First, we provide a precise definition of the hidden progress measure for deep learning models. The focus is on formalizing the qualitative description of this concept presented in \cite{barakHiddenProgressDeep2022}.

\begin{definition}[Progress Measure]
    Let the complete set of parameters of a deep learning model be denoted by $M$, and the update step by \( \displaystyle n \). Given a function $f:M \times \mathbb{Z}_{+} \to \mathbb{R}$, if there exists a mapping $\varphi \circ f:W \times \mathbb{Z}_+ \rightarrow \{0,1\}$ such that $\varphi \circ f$ takes the value of 1 when a specific phenomenon occurs and 0 otherwise, then we call $f$ a \textbf{progress measure} of this specific phenomenon of the model. 
\end{definition}

Through this strict definition, it becomes clear that a progress measure is an explanation-based monitoring value for detecting which quantities change when the model groks. In other words, it tracks a trait from which the test loss trend can be inferred through the model parameters. Consequently, the design of such a measure is tightly linked to the fundamental reason the model can generalize.

However, an inherent limitation exists in current progress measures. They merely reveal properties of model parameters within the minimal region defined by the weight decay term, without proving that these properties are necessary for the weight decay term to reach a local minimum.

From the discussion in Section~\ref{math_analysis}, the necessary and sufficient condition for the weight decay term to reach a local minimum is that the embedding in the equivalence class of the training set remains uniform. This insight explains why certain progress measures, despite not delving deeply into theory, can still succeed: they indirectly capture some aspect of uniformity. For instance, when Fourier decomposition exposes significant weights, it indicates periodic signals, which is a classical concept in dynamical systems~\cite{brinIntroductionDynamicalSystems2002}.

Meanwhile, interpretable progress measures are often cumbersome to compute. By directly applying our theorem, however, a concise progress measure can be constructed, termed the main embedding diff (MED), which straightforwardly captures changes in test accuracy after training accuracy reaches 1.

\begin{definition}\label{med}
    For the epoch $n$, the main embedding diff (MED) of $n$ is defined by
    \begin{equation}
        \mathsf{MED}(n)=\frac{1}{p} \sum_{i=0}^{p-1} \| \phi^{(n)}(i+1)-\phi^{(n)}(i)\|,
    \end{equation}
    where $\phi^{(n)}$ denotes the embedding mapping $\phi$ after $n$ epochs of updates , $p$ denotes the selected prime number.
\end{definition}

\begin{remark}
    The value of the $\mathsf{MED}$ function and the test loss exhibit nearly identical changes, as we will demonstrate in the experimental section.
\end{remark}

\section{Separability}
\label{region}

In this section, we extend the previous approach by replacing linear separation with convex hull separation of sample points. Our earlier theoretical work focused on the model’s embedding uniformity. Here, we point out that striving for uniformity is akin to seeking a stable decision boundary in the embedding space, and we introduce a convex hull perspective to extend this idea into a more general theory of decision separation.

Among the many progress measures proposed in the past, one stands out for its geometric intuition: it characterizes the movement of sample points under the influence of weight decay. This approach, known as local complexity, is based on linear regions~\cite{humayun2024deep,Hanin2019deep}. Essentially, once the network structure is fixed, the classification head induces a partition of the embedding space, following the network’s inverse mapping. In other words, \ul{\textit{neural network training can be viewed as the evolution of complex decision boundaries in the embedding space, and generalization is reflected by the movement of embeddings corresponding to non-training samples within this space}}. In the case of DNNs, this partition corresponds to the linear regions. Initially, the model creates a partition encompassing all training sample points. At this stage, the weight decay term pushes test samples toward similarly labeled training samples. Local complexity serves as a progress measure that captures this movement process.

From this perspective, the deep learning model effectively learns the separability of a low-dimensional manifold within a high-dimensional Euclidean space. In other words, its learning objective can be viewed as the linear sum of characteristic functions defined over a complex region:
\begin{equation}
    f(x) = \sum _{r \in R} \theta_r \mathbf{1}_{\Omega_r},
\end{equation}
with $R$ represents the set of indicators of the continuous region that needs to be learned, $\Omega_r$ represents the corresponding continuous region with $\mathbf{1}_{\Omega_r}$ represents its characteristic function, and $\theta_r$ represents the characteristic assigned by the model. The classic results of analysis tell us that this type of function is dense in the Lebesgue measurable function set on $\mathbb{R}^n$. Therefore, this understanding is very appropriate.

The real difficulty lies in how to measure this separation. We can only assume that the model stores knowledge in the form of \ul{\textit{a continuous manifold of the convex hull that can contain multiple training sample points}}. Then, the ability to complete this separation is measured by the convex separability of the embedded point set. The goal of deep learning is to learn the manifold of the embedding space that needs to be separated from the convex hull of discrete samples. We define three types of high-dimensional point set separability to describe this process:





\begin{definition}
    Let $A$ and $B$ be two finite point sets in $n$-dimensional Euclidean space. We call $A$ and $B$ 1) \textit{topologically separable} if and only if $A \cap B = \emptyset$; 2) \textit{convexly separable} if and only if $\mathsf{Conv}(A) \cap \mathsf{Conv}(B) = \emptyset$, where $\mathsf{Conv}(A)$ and $\mathsf{Conv}(B)$ denote the convex hulls of $A$ and $B$, respectively; and 3) \textit{convexly separable on subspace} $\mathcal{U}$ if and only if the projections of $A$ and $B$ onto $\mathcal{U}$ are convexly separable.
\end{definition}

Among them, convex separability is stronger than subspace convex separability, and topological separability is stronger than positional separability. The role of positional embedding is to provide topological separability of the same vocabulary. The goal of transformer learning is to make the embedding convexly separable on a certain subspace.

We hypothesize that the presence of bilinear operators in the transformer causes the division of regions to deviate from linearity. To address this, we adopt the convex separability of the subspace as an alternative to linearity. Furthermore, we conjecture that as the number of transformer layers increases, the growth of regional differentiation will follow at least a quadratic polynomial form. In other words, the linear combination of the learned characteristic functions will take the following form:
\begin{equation}
      f(x) = \sum _{r \in R} \theta_r (\sum _{t \in R_r} \mathbf{1}_{\Omega_t}).
\end{equation}

We verify this conjecture on the grokking task in Figure~\ref{layer}. As the number of layers increases, due to the rapid decrease in the volume of the indicative area, the test sample points will enter and exit the corresponding indicative area differently, corresponding to the test accuracy rate oscillating sharply in a large range.  This fully demonstrates the subdivision of the region. In our prime field operation task, this means that the continuous regions representing the congruence of the operation results become more numerous and smaller.

\begin{figure}[t!]
    \centering
    \centerline{\includegraphics[width=0.90\columnwidth]{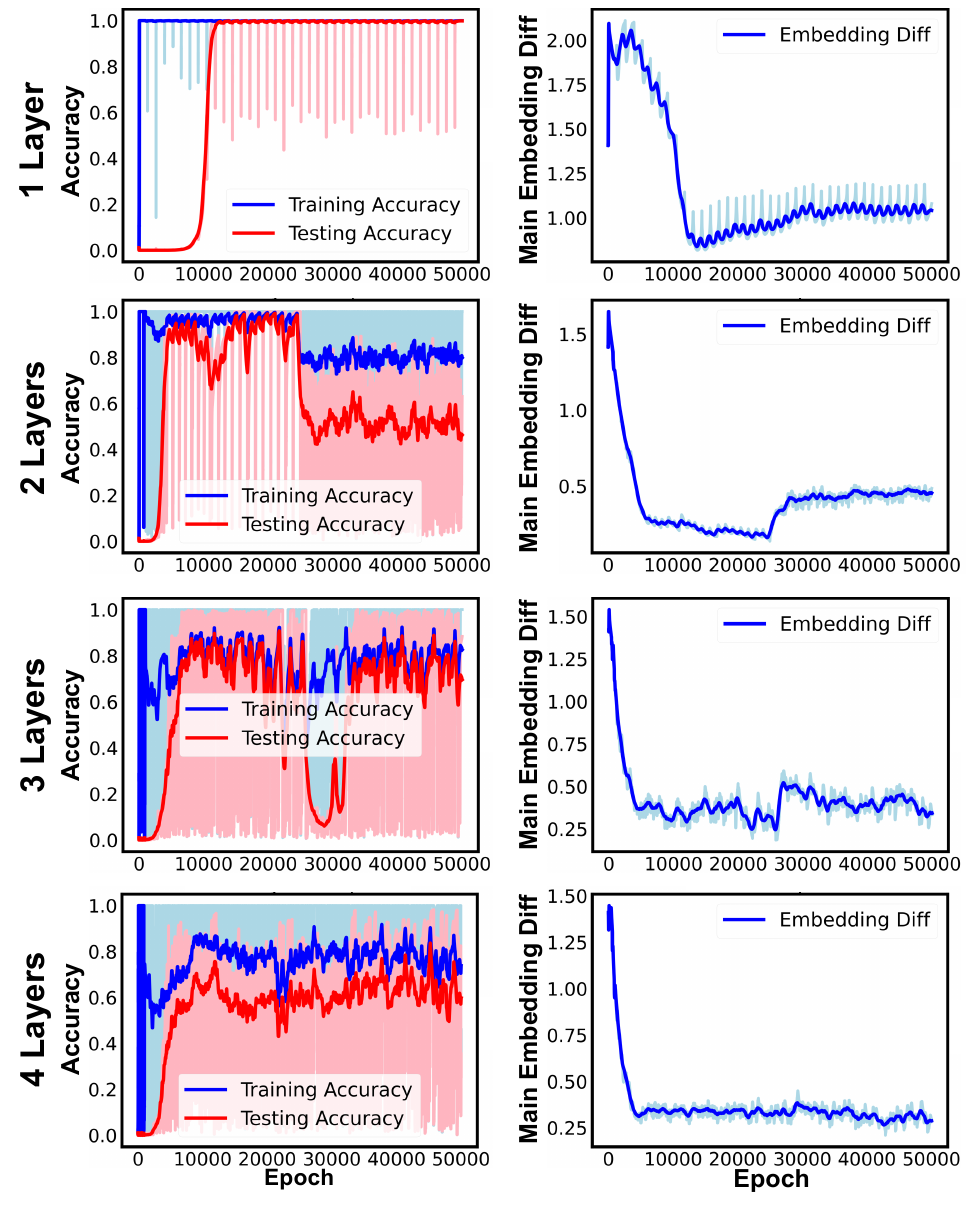}}\vspace{-12pt}
    \caption{\textbf{Impact of Using More Transformer Layers.} As the number of transformer layers increases, the test accuracy fluctuates violently, while the med keeps decreasing.}
    \label{layer}
    \vspace{-16pt}
\end{figure}

\section{Experiment}
\label{experiments}
We have rigorously shown that the essence of the grokking phenomenon lies in the uniformity induced by the weight decay term, which enables the training set to generalize to surrounding data points. Moreover, because of the specific nature of the task, only a small amount of data is required to achieve this effect. In Section~\ref{expacc}, we demonstrate a corollary of the insight presented in Section~\ref{acc}: by designing a specialized dataset structure, it is possible to control the upper bound of test accuracy reachable through grokking. Section~\ref{expmed} verifies that our proposed progress measure function changes in tandem with the test loss. Finally, in Section~\ref{expresnetgrokk}, grounded in our theoretical framework, we construct a task that induces grokking in ResNet-18, broadening previous observations that focused on shallower architectures such as 1--2 layer DNNs or transformers.

\subsection{Controlling the Upper Limit of Test Accuracy}
\label{expacc}

We have emphasized in Section~\ref{acc} that the improvement in test accuracy depends not only on the uniform embedding induced by the weight decay term, but also on the distribution of the training set over the entire dataset. When the proportion of the training set exceeds a certain threshold, the probability of the corresponding test accuracy reaching 1 is 100\%. Consequently, by constructing a dataset with a particular distribution, it is possible to modify the upper limit of test accuracy achievable through grokking.

The experimental settings are as follows. Based on the original power experiment, let \(p=97\), set the training set proportion to 0.3, and choose a weight decay coefficient of 1. Under these conditions, when the training set is randomly sampled, the test accuracy begins to rise steadily around epoch 5000 and eventually reaches 1.

Next, the training set is constructed in a targeted manner according to Theorem~\ref{thm:dist}. First, a specific portion of the data is set aside for the test set. Then, from the remaining data, 30\% is randomly selected as the training set, and the remainder is again added to the test set. Conceptually, if the entire dataset is viewed as a two-dimensional matrix, this process can be seen as removing certain strips or square subregions from the matrix. We regard the entire data set as a $p \times p$ grid, where the blue grid represents the corresponding coordinates belonging to the training set, and the gray grid represents the corresponding coordinates belonging to the test set. An example of this dataset selection scheme is illustrated in Figure~\ref{distrubution}.

\begin{figure}[t]
    \centering
    \centerline{\includegraphics[width=\columnwidth]{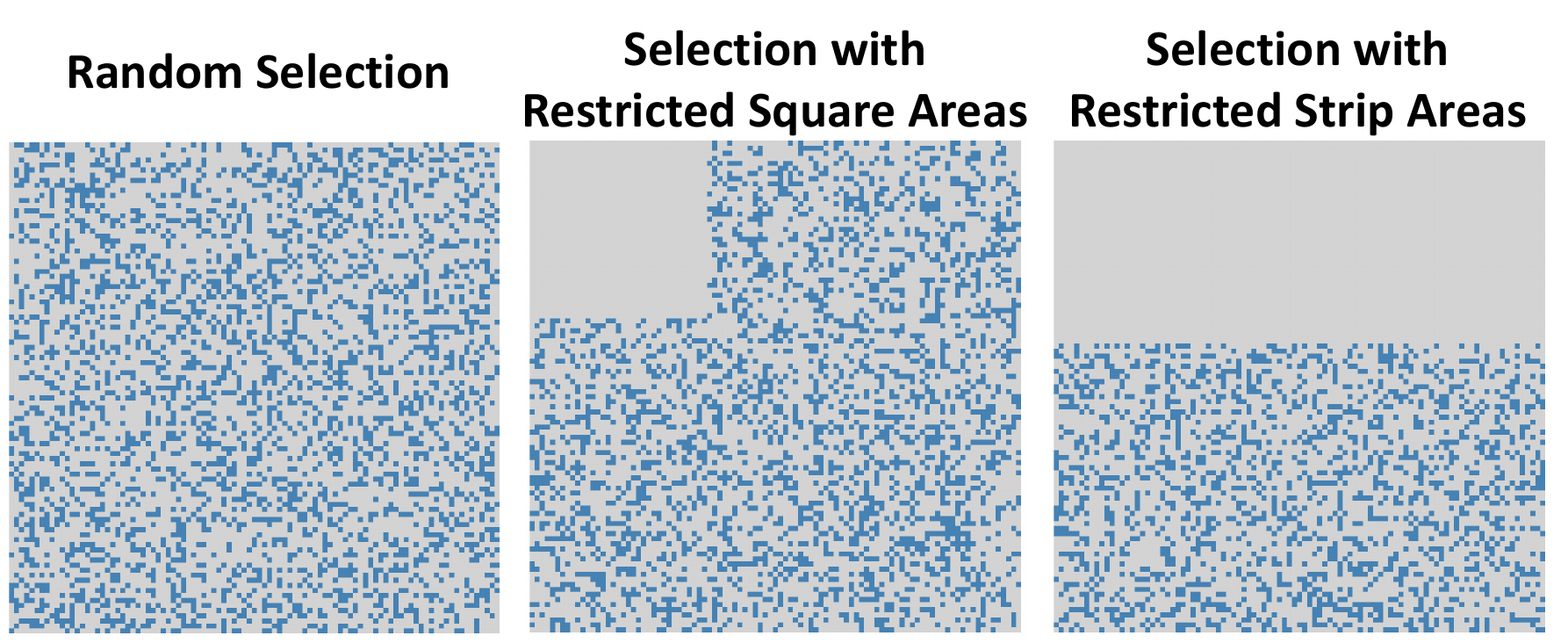}}\vspace{-10pt}
    \caption{\textbf{Designed Training Set Distribution.} The first picture is randomly
    selected, the second picture corresponds to the square area we dug
    out, and the third picture corresponds to the strip area we dug out.}
    \label{distrubution}
    \vspace{-14pt}
\end{figure}

We summarize the experimental results in Table~\ref{table1}. For each specified training set distribution, we record the time (in epochs) at which grokking occurs as GO-Epoch, the upper limit of accuracy achievable through grokking as U-Acc, and the lowest accuracy observed after reaching this upper limit as L-Acc. The two values illustrate the range of accuracy oscillation following grokking.

\begin{table}[t]
    \caption{Upper Limit of Accuracy Under Specific Training Set Distributions.}
    \label{table1}
    \centering
    \begin{small}
    \begin{sc}
    \begin{tabular}{c|cccc}
    \toprule
     Method & Para & U-Acc & L-Acc & GO-Epoch \\
    \midrule
    \multirow{7}{*}{\shortstack{SQUARE \\ \textit{$i=0,1,...,k$} \\ \textit{$j=0,1,...,k$} }}
                            & $k$=30 & 99.04 & 91.27 & 1500($\pm$100) \\
                            & $k$=35 & 91.07 & 80.73 & 1500($\pm$100) \\
                            & $k$=40 & 84.50 & 81.68 & 1700($\pm$100) \\
                            & $k$=45 & 79.03 & 76.07 & 1900($\pm$100) \\
                            & $k$=50 & 77.39 & 66.02 & 3400($\pm$200) \\
                            & $k$=55 & 62.09 & 57.69 & 2600($\pm$100) \\
                            & $k$=60 & 57.31 & 55.18 & 2200($\pm$100) \\ \midrule
    \multirow{4}{*}{\shortstack{STRIP \\ \textit{$i=0,1,...,t$} \\ \textit{$j=0,1,...,97$} }}
                            & $t$=67 & 7.56& 3.58& disappear \\
                            & $t$=60 & 33.25 & 17.05 & 900($\pm$100) \\
                            & $t$=50 & 80.10& 73.61& 1600($\pm$100) \\
                            & $t$=40 & 89.89 & 84.80 & 1300($\pm$100) \\
    \bottomrule
    \end{tabular}
    \end{sc}
    \end{small}
\vspace{-10pt}
\end{table}

In fact, what we are discussing is still a necessary condition. The upper limit of the grokking test accuracy of this task is determined by the distribution of the training set in each equivalence class in $I \times J / \sim$. It is just that when $f(i,j)=i+j$, the distribution of each of these equivalence classes is completely isomorphic. From an algebraic geometry perspective, \ul{\textit{the distribution of a dataset is determined by the distribution of integer solutions of an algebraic curve taken modulo $p$}}. However, this problem does not admit a general closed-form solution. Of course, we can change the expression of $f(i,j)$ to make the distribution of each of these equivalence classes complex. At this time, even if the value is randomly selected, the upper limit of the grokking test accuracy cannot reach 1, for example, we take $f(i,j)=i^2+ij+j^2$. The specific experimental results are shown in Figure~\ref{x2}. We use $\text{frac}$ to represent the training set proportion. When \( \text{frac} = 0.5, 0.6, 0.7 \), the relationship \( \text{Acc}(1 - \text{frac}) = \text{Constant} \) holds approximately. Therefore, the generalization ability of the model has not improved as a result of the increased training data.
\begin{figure}[ht]
    \vspace{-16pt}
    \centering
    \centerline{\includegraphics[width=0.75\columnwidth]{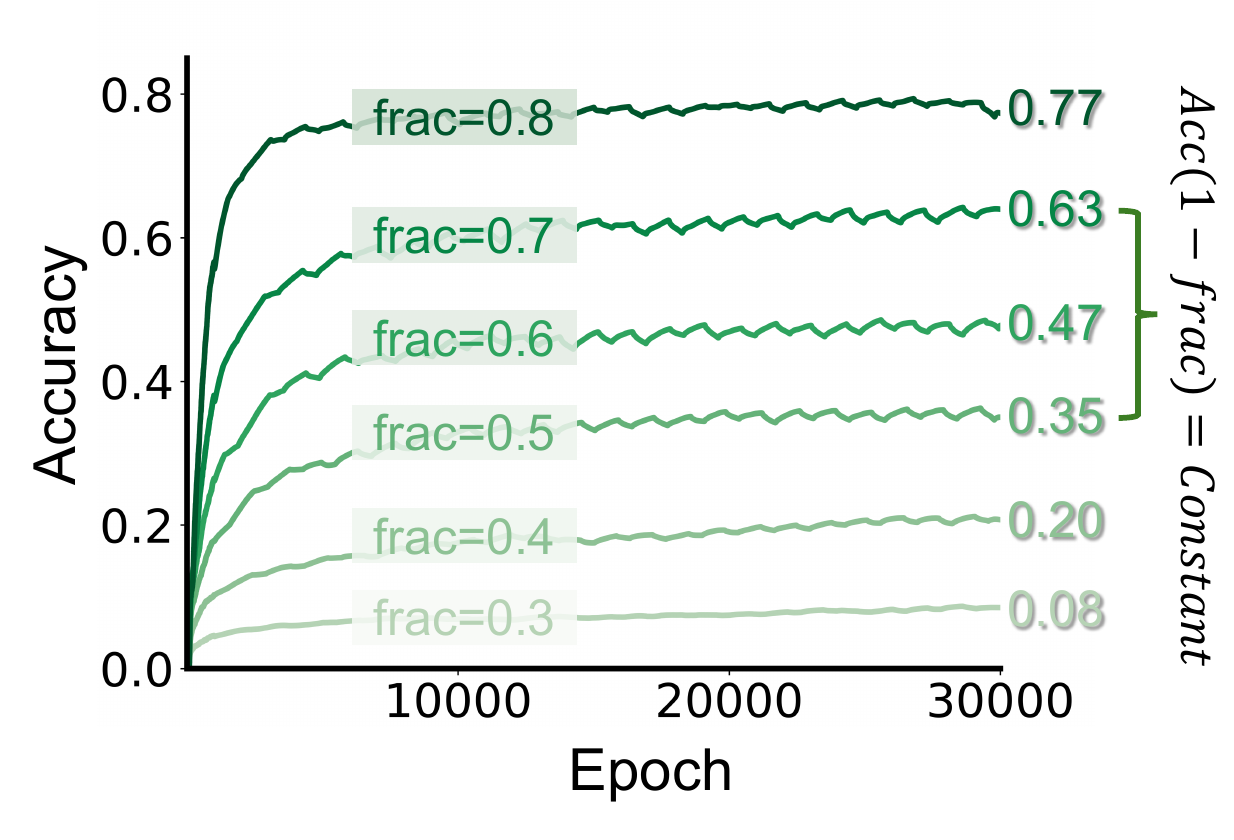}}\vspace{-10pt}
    \caption{\textbf{Accuracy Curves When \( f(i,j) = i^2 + ij + j^2 \).} The upper limit of accuracy increases as \( \text{frac} \) increases. However, the number of samples that can be correctly output in the test set has not increased}
    \vspace{-16pt}
    \label{x2}
\end{figure}

We ultimately find that the difference between DNNs and Transformers in this task lies in their specialized operational mechanisms. Both can handle operations conforming to linear structures, such as \(f(i, j) = a i^n + b j^m\). However, in the above experiment, if we replace the 1-layer transformer with a 2-layer DNN~\cite{gromov2023grokkingmodulararithmetic}, the test accuracy will always remain around 0.01. DNNs completely fail on \(f(i, j) = i^2 + i j + j^2\), whereas Transformers retain a degree of generalization. The detailed results and corresponding experiments are provided in Appendix~\ref{support}.

\subsection{Synchronous Changes of MED and Test Loss }
\label{expmed}
In Section~\ref{progress measure}, we discussed the progress measure method for monitoring grokking. In fact, the essence of this monitoring is to find the weight decay term as a feature of the update direction of the loss function. In Theorem~\ref{thm:med}, we proved that the uniformity of embedding in the feature direction is the direction in which the weight decay decreases. Therefore, we can construct a progress measure function MED that can be easily calculated as shown in Definition~\ref{med}. The MED function's variation curve closely matches that of the task's test loss.

We maintain the same configuration as in Section~\ref{expacc}, randomly selecting 30\% of the entire dataset for the training set. The only difference is that we set \(p=47, 97, 211, 379\). We then plot the corresponding Loss, MED, and Accuracy curves to illustrate how effectively our progress measure monitors the test loss. The specific experimental results are shown in Figure~\ref{MED}.


\begin{figure}[ht]
    \begin{center}
    \centerline{\includegraphics[width=\columnwidth]{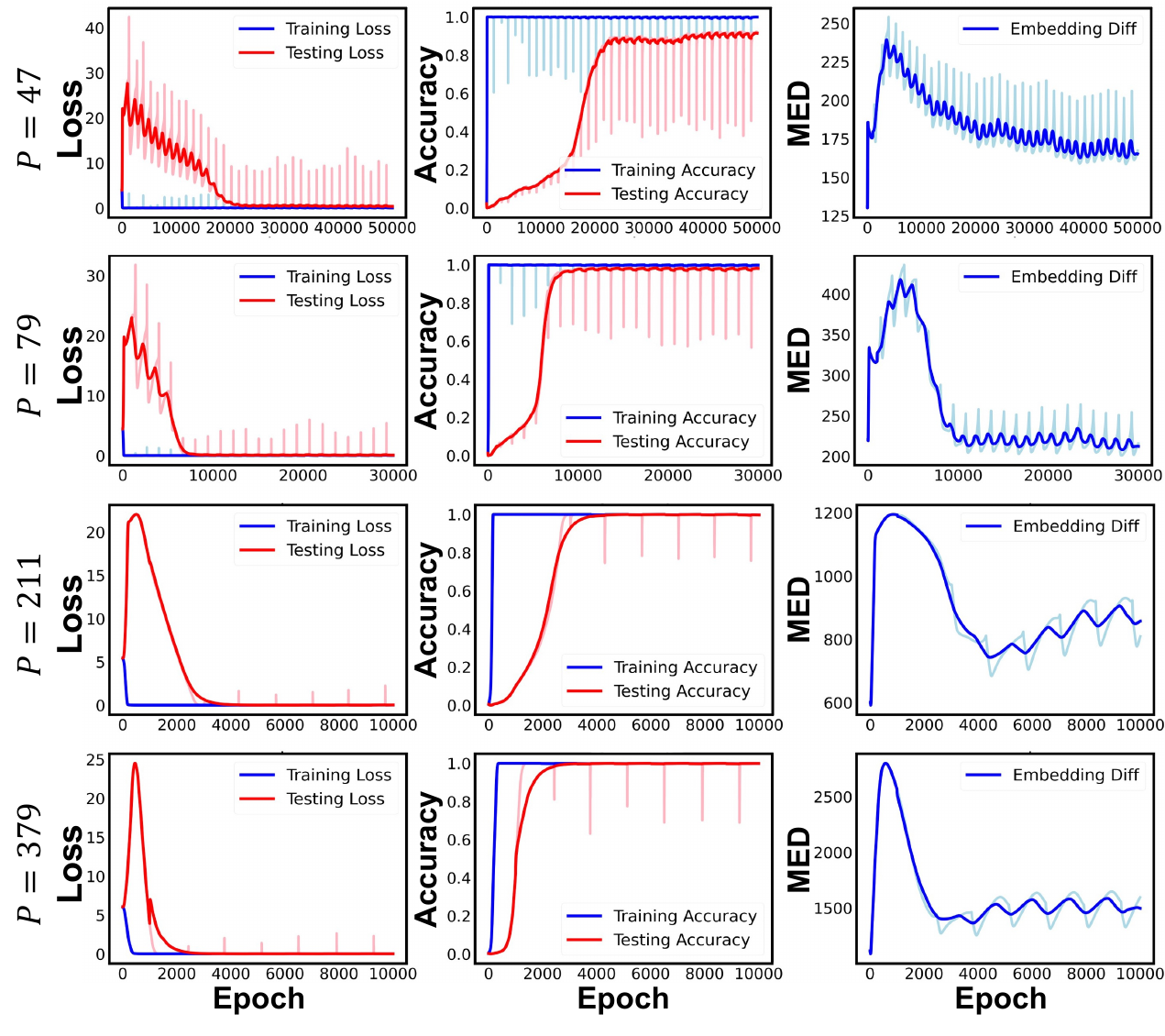}}\vspace{-12pt}
    \caption{\textbf{Performance of MED under different prime numbers.} To amplify the changes in the MED, we omitted the \( \frac{1}{p} \) coefficient defined in Definition~\ref{med} which means we did not take the average here. It can be observed in all four groups of experiments that the MED and test loss have consistent changing trends.}
    \label{MED}
    \end{center}
    \vspace{-32pt}
\end{figure}

\subsection{A Task Grokking on Resnet-18}
\label{expresnetgrokk}
In this experiment, the following issue is addressed: grokking is not exclusive to prime-field operations, nor is the transformer the only architecture that exhibits grokking. Indeed, our proof shows that the transformer itself is not strictly necessary. The critical insight is that neither the architecture nor the task solely determines the occurrence of grokking; rather, it is governed by specific properties of both the architecture and the task.

According to our theorem, the necessary condition for grokking is that the task’s input exhibits a suitable representational structure and that the network has a corresponding property structure. In the prime-field operation task, this representational structure is a linear periodic structure. As is well-known, convolutional networks tend to destroy linearity due to pooling layers, which explains the lack of a prominent grokking phenomenon in these networks. Guided by our theoretical understanding and the features of convolution, we construct a dataset that ultimately induces grokking in a ResNet-18.

To create a dataset that induces grokking in ResNet-18, the shallow convolutional layers are configured to extract local features, the deep convolutional layers perform feature operations, and the deeper layers maintain linear features. The task is defined as follows: first, \textbf{a dictionary for image classification} is constructed by recording \(n\) images and their corresponding categories, with categories represented by integers from 1 to \(n\). Each image is then divided into four parts, with each part embedded into an image from the dictionary such that the category of the resulting image is the sum of the categories of the four regions. This process generates a dataset of size \(n^4\). Detailed instructions are provided in Figure~\ref{resnetdata}.

\begin{figure}[ht]
    \vspace{-4pt}
    \centering
    \centerline{\includegraphics[width=\columnwidth]{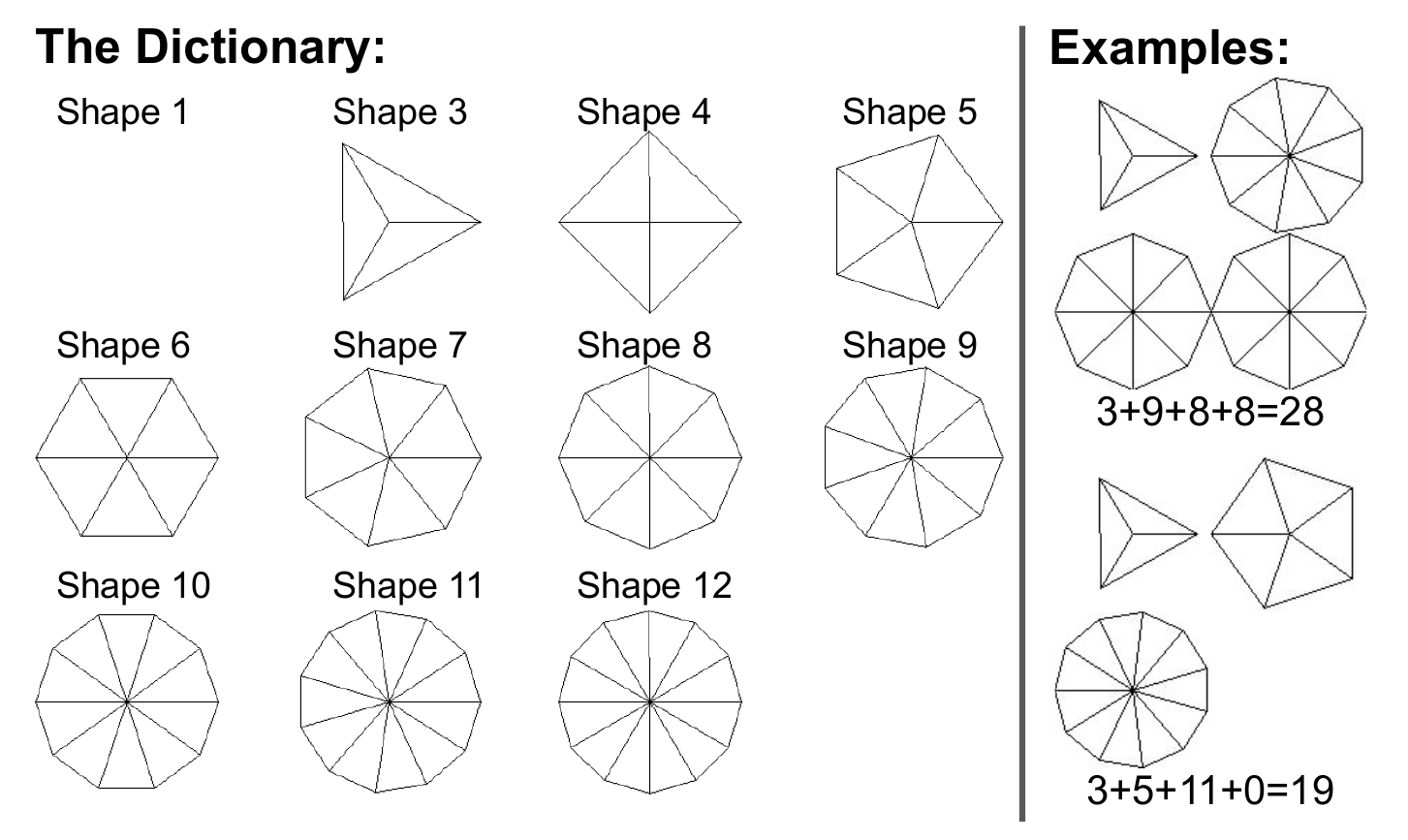}}\vspace{-12pt}
    \caption{\textbf{Structure of the Data Set.} The left side of the figure illustrates our dictionary, while the right side provides an example of a generated sample.}
    \label{resnetdata}
\end{figure}

\begin{figure}[h]
    \vspace{-12pt}
    \centering
    \centerline{\includegraphics[width=0.9\columnwidth]{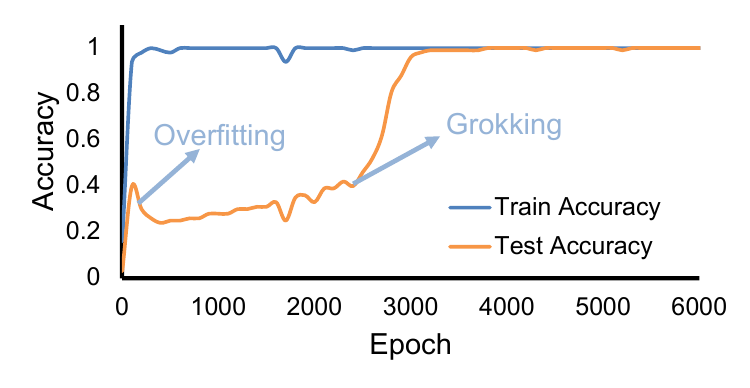}}\vspace{-18pt}
    \caption{\textbf{The Grokking Phenomenon on Resnet-18.} The training accuracy reaches 1 after approximately 100 epochs, while the test accuracy fluctuates between 0.3 and 0.4 for about 2800 epochs before rapidly rising to 1. }
    \label{resnet-fig}
    \vspace{-14pt}
\end{figure}

We use the ResNet-18 model implemented in PyTorch, the Adam optimizer, a learning rate of \(10^{-3}\), and a weight decay coefficient of \(10^{-4}\). The dataset is configured as described in Figure~\ref{resnetdata}. The dictionary contains 11 entries in total, allowing for the generation of 14,641 samples with 25\% of the data randomly selected as the training set. The experimental results are presented in Figure~\ref{resnet-fig}.




\section{Conclusion}
In this paper, we systematically studied the grokking phenomenon of computational tasks on prime domains and analyzed the reasons behind the increase in test accuracy. We decomposed this increase into two parts: the uniformity brought about by weight decay, and the further accuracy improvement resulting from the interaction of that uniformity with the training set distribution. We proved these points rigorously and verified them through detailed experiments. We also noted that this uniformity should be interpreted as representation uniformity rather than a concept restricted to the embedding space, implying that both model architecture and task jointly determine how uniformity manifests. Based on this, we proposed an effective progress measure for prime-domain tasks and designed a ResNet-18 grokking task, extending beyond prior work on shallow networks and Transformers. Our findings suggest that for cognitively bounded domains, structured data can be more beneficial than larger data volumes.


\section*{Impact Statement}

This work aims to improve the understanding of the special generalization behavior of deep networks. This paper points out that training sets with good structure distribution can achieve better generalization effect under the effect of weight decay. Therefore, it may affect the subsequent performance-oriented models to increase the number of training rounds, thus resulting in a waste of computing resources.

\nocite{brinIntroductionDynamicalSystems2002}

\bibliography{example_paper}
\bibliographystyle{icml2025}


\newpage
\appendix
\onecolumn

\section{Notations}

The notations used throughout this article are summarized in Table \ref{table:notations}.

\begin{table}[h!]
\centering
\caption{Some important notations used in this paper.}
\label{table:notations}
\begin{tabular}{@{}ll@{}}
\toprule
\textbf{Notation} & \textbf{Description} \\ \midrule
$mod$ & The remainder when an integer is divided by p \\
$\mathbb{Z}_p$ & The field of prime number p $\{0,1,2...,p\}$\\
$\mathbb{Z}_{+} $ & The set of positive integers\\
$f(i,j)$ & A two-variable polynomial function\\
$r(i,j)$ & The remainder of $f(i,j)$ modulo $p$ by $r(i,j)$\\
$I_d$ &  Identity operator, $I_d(x)=x$\\
$l,l_O,l_V,l_1,l_2$ &  Linear operator\\
$n$ & Normalization operator \\
$r$ & Truncation operator, set the components in a specific subspace to 0\\
$\beta_{QK}$ & Bilinear operator \\
$X$ & Input token collection\\
$\phi$ & The embedding function which combines word embedding and position embedding\\
$\phi^{(n)}$ & Embedding function after $n$ epochs' updating\\
$\times$ & Cartesian product\\
$I \times J$ &  The Cartesian product of $I$ and $J$, representing the entire data set\\
$ \sim $ & A equivalence relation in sets\\
$S/\sim$ & The quotient of the set $S$ with respect to the equivalence relation $\sim$, whose elements are subsets of $S$\\
$I' \times J'$ &  The Cartesian product of $I'$ and $J'$, representing the training data set\\
$I'' \times J''$ &  The Cartesian product of $I''$ and $J''$, representing the test data set\\
$c_i$ &  The vector representing the optimization direction\\
$\epsilon_i$ &  The acceptable error distance of the model\\
$\| \cdot \|$ & Operator norm \\
$\| \cdot \|_F$ & Frobeniums norm \\
$W$ & The weight decay term\\
$E$ & The critical value of weight decay\\
$\Omega$ & Formal definition, set of all parameters when $W<E$\\
$P(\cdot)$ & Probability Function \\
$\alpha,\text{frac} $ & Training set ratio\\
$\text{prop}$ & The proportion of the set that the training set can cover within a certain Manhattan distance in the overall data set\\
$E[\cdot]$ & Mathematical expectation\\
$R,R_r$ & the set of labels\\
$\theta_r$ & The value corresponding to the label r\\
$\Omega_r$ & The set corresponding to the label r\\
$\mathbf{1}_{\Omega_r}$ & The characteristic function of set $\Omega_r$\\
$Conv(A)$ & The convex hull of set A\\
\bottomrule
\end{tabular}
\end{table}

\section{Proof of Theorem in Section~\ref{math_analysis}}
\subsection{Proof of Lemma~\ref{lem:softmax}}
\begin{proof}
If the model produces the correct output, then
\[
\mathsf{Softmax}(\bar{c}_{i,j}) - r(i,j) = 0.
\]
Let
\[
V_i = \{\,x \in \mathbb{R}^n \mid x_i > x_j \;\text{for all}\; j \neq i\}.
\]
It follows that \(\bar{c}_{i,j} \in V_{r(i,j)}\). When \(n\) is sufficiently large, the orthogonality of \(\bar{c}_{i,j}\) becomes pronounced, making it possible to find a suitable vector within \(V_i\). Concretely, denote by \(B_r\) the inscribed sphere of \(V_i\) with radius \(r\). By applying an appropriate regular mapping, all the \(\bar{c}_{i,j}\) can be placed within this sphere \(B_r\). In that case, the center and radius of \(B_r\) are precisely the vectors \(c_i\) and \(\epsilon_i\) we seek. Note that they are not unique, as there may be other valid choices for the center-radius pair.

Therefore, this lemma essentially states that tokens belonging to the same class are ultimately drawn into the same cone.
\end{proof}

\subsection{Proof of Lemma~\ref{lem:norm}}
\begin{proof}

The operator norm of $l$ is defined as:
\[
\|l\| = \sup_{\|x\|_2 \neq 0} \frac{\|l(x)\|_2}{\|x\|_2} = \sup_{\|x\|_2 = 1} \|Mx\|_2,
\]
where $M$ is the matrix representation of $l$.

The Frobenius norm of $M$ is:
\[
\|M\|_F = \sqrt{\sum_{i=1}^m \sum_{j=1}^n |M_{ij}|^2}.
\]

For any vector $x \in \mathbb{R}^n$, we have:
\[
\|Mx\|_2^2 = \sum_{i=1}^m \left( \sum_{j=1}^n M_{ij} x_j \right)^2.
\]

By the Cauchy-Schwarz inequality:
\[
\|Mx\|_2^2 \leq \sum_{i=1}^m \left( \sum_{j=1}^n M_{ij}^2 \right) \|x\|_2^2 = \|M\|_F^2 \|x\|_2^2.
\]

Taking the square root:
\[
\|Mx\|_2 \leq \|M\|_F \|x\|_2.
\]

For $\|x\|_2 = 1$, taking the supremum over $x$:
\[
\|l\| = \sup_{\|x\|_2 = 1} \|Mx\|_2 \leq \|M\|_F.
\]

The Frobenius norm of $M$ can be expressed in terms of its singular values. Let the singular value decomposition of $M$ be:
\[
M = U \Sigma V^\top,
\]
where $U \in \mathbb{R}^{m \times m}$ and $V \in \mathbb{R}^{n \times n}$ are orthogonal matrices, and $\Sigma$ is a diagonal matrix with singular values $\{\sigma_i\}$.

The Frobenius norm is:
\[
\|M\|_F^2 = \sum_{i=1}^{\min\{m, n\}} \sigma_i^2.
\]

The operator norm is the largest singular value:
\[
\|l\| = \|M\| = \sigma_{\max}.
\]

Thus:
\[
\|M\|_F^2 = \sum_{i=1}^{\min\{m, n\}} \sigma_i^2 \leq \min\{m, n\} \cdot \sigma_{\max}^2 = \min\{m, n\} \cdot \|l\|^2.
\]

Taking the square root:
\[
\|M\|_F \leq \sqrt{\min\{m, n\}} \|l\|.
\]

Combining the results:
\[
\|l\| \leq \|M\|_F \leq \sqrt{\min\{m, n\}} \|l\|.
\]
\end{proof}
\subsection{Proof of Theorem~\ref{thm:uniemb}}
\begin{proof}

Recall definitions in Section~\ref{math_analysis} that
\[
\bar{\phi}(\mathsf{cls})
= I_d(\phi(\mathsf{cls})) 
  + l_1 \circ r \circ l_2(\phi(\mathsf{cls})) 
  + \bigl(I_d + l_1 \circ r \circ l_2\bigr)\circ l_O \circ 
    \Bigl(
       n \circ \Bigl(\sum_{t\in \{i,j,\mathsf{cls}\}}\beta_{QK}\bigl(\phi(t),\phi(\mathsf{cls})\bigr)\,l_V\bigl(\phi(t)\bigr)\Bigr)
    \Bigr).
\]
Removing the constant (invariant) terms from this expression yields:
\[
\bar{\phi}(0)
= \bigl(I_d + l_1 \circ r \circ l_2\bigr)\circ l_O \circ 
    \Bigl(
       n \circ \bigl(\beta_{QK}\bigl(\phi(i),\phi(\mathsf{cls})\bigr)\,l_V\bigl(\phi(i)\bigr)
                   + \beta_{QK}\bigl(\phi(j),\phi(\mathsf{cls})\bigr)\,l_V\bigl(\phi(j)\bigr)
              \bigr)
    \Bigr).
\]

Since addition and subtraction of linear transformations remain linear, we write:
\[
\bar{\phi}(\mathsf{cls})
= r' \circ l'\bigl(\phi(i)\bigr) \;+\; r'' \circ l''\bigl(\phi(j)\bigr),
\]
where \(r'\) and \(r''\) are truncation operators, and \(l'\) and \(l''\) are linear operators.

Owing to the widespread usage of normalization layers, it is reasonable to assume
\[
\|l'\| \;\approx\; \|l''\|.
\]

From Lemma~\ref{lem:softmax}, it follows that
\[
\bigl\|\,r' \circ l'\bigl(\phi(i)\bigr) \;+\; r'' \circ l''\bigl(\phi(j)\bigr)
      \;-\; c'_{r(i,j)}\bigr\|
< \epsilon'_{r(i,j)}.
\]
This directly implies
\[
\|\;\phi(i) \;+\; \phi(j) \;-\; k_{r(i,j)}\|\;<\;\delta_{r(i,j)},
\]
for appropriate choices of \(k_{r(i,j)}\) and \(\delta_{r(i,j)}\).

As the size of \(A \times B\) grows, the set of admissible vectors \(\{k_i\}\) converges toward a single point, and the corresponding \(\delta_i\) values shrink accordingly. Hence, while \(k_i\) and \(\delta_i\) are not strictly unique for finite \(|A \times B|\), they become more tightly constrained as \(|A \times B|\) increases.

\emph{In essence, this argument establishes that if the model converges correctly, the linear operations (subject to truncation and normalization) enforce an approximate uniform embedding in the deeper layers.}
\end{proof}
\subsection{Proof of Theorem~\ref{thm:med}}
\begin{proof}
    We denote the set of linear operators involved in the model as $L$. According to Lemma~\ref{lem:norm} and Theorem~\ref{thm:uniemb}, we can regard the model update after the training accuracy reaches 1 as the following optimization problem:
    \[
    min \sum_{l \in L} \|l\|
    \]
    with
    \[
    \bigl\|\,r' \circ l'\bigl(\phi(i)\bigr) \;+\; r'' \circ l''\bigl(\phi(j)\bigr)
      \;-\; c'_{r(i,j)}\bigr\|
< \epsilon'_{r(i,j)}
    \]
    holds for any element in $I \times J/\sim$.

Even with operator simplification, this problem is still extremely complicated due to the constraints. To address this question, we first consider the optimization problem under the constraints provided only for one element in $I \times J /sim$.

\begin{lemma}\label{lem:bounded-norm-min}
Let $V,W$ be real Euclidean spaces, let $l: V \to W$ be a linear map, and fix a prime $p \ge 2$. Suppose we have vectors $x_0, x_1, \ldots, x_{p-1} \in V$ satisfying
\[
   m \;\le\; \|x_i\| \;\le\; M 
   \quad \text{for all }0 \le i \le p-1,
\]
where $0 < m \le M < \infty$. Let $r \in W$ be a nonzero vector. Assume the constraint
\[
   i + j \equiv 0 \;(\bmod\; p)
   \quad \Longrightarrow \quad
   l(x_i) + l(x_j) \;=\; r.
\]
Then there exists a linear map $l$ satisfying the above constraint for which the operator norm $\|l\|$ (in the Euclidean/2-norm sense) attains its minimum possible value, and
\[
   \min \bigl\|l\bigr\|
   \;=\;
   \frac{\|r\|}{2\,M}.
\]
\end{lemma}

\begin{proof}
    First, note that for each pair $(i,j)$ with $i+j \equiv 0 \pmod{p}$, we must have
\[
   l(x_i) + l(x_j)
   \;=\;
   r,
\]
in particular, taking $i,j=0$ yields $2\,l(x_0) = r$, so $l(x_0) = r/2$. For $i \neq 0$, we require $l\bigl(x_i + x_{p-i}\bigr) = r$. Hence each sum $x_i + x_{p-i}$ must be mapped to $r$.  

To minimize $\|l\|$, we seek to maximize the norm of any input $v \in V$ for which $l(v) = r$. Here, $v$ can be chosen as $x_i + x_{p-i}$. By the bound $\|x_i\|\le M$, we know
\[
   \bigl\| x_i + x_{p-i} \bigr\|
   \;\le\;
   \|x_i\| + \|x_{p-i}\|
   \;\le\;
   M + M
   \;=\;
   2\,M.
\]
Thus the largest possible norm of $x_i + x_{p-i}$ is $2M$ if we choose $\|x_i\|=\|x_{p-i}\|=M$ and align them in the same direction.  

A direct construction is: pick a single nonzero vector $x^\star \in V$ with $\|x^\star\| = M$ and set $x_i = x^\star$ for all $i$. Then $x_i + x_{p-i} = 2\,x^\star$, which has norm $2M$. Define a rank-1 operator $l$ by
\[
   l(w)
   \;=\;
   \frac{\langle w,\,2x^\star\rangle}{\|2x^\star\|^2}\,r
   \;=\;
   \frac{\langle w,\,2x^\star\rangle}{4\,M^2}\,r.
\]
Since $\|2x^\star\| = 2M$, we have $l(2x^\star) = r$ and therefore $l(x_i)+l(x_{p-i})=l(2x^\star)=r$ for all pairs $(i,p-i)$.  

Finally, we verify that this construction indeed minimizes $\|l\|$. For any linear operator $l$, if $l(v)=r$, then
\[
   \|l\|
   \;=\;
   \sup_{w \neq 0}\frac{\|l(w)\|}{\|w\|}
   \;\ge\;
   \frac{\|l(v)\|}{\|v\|}
   \;=\;
   \frac{\|r\|}{\|v\|}.
\]
Since $\|v\|\le 2M$ for our constraint, the infimum of $\|l\|$ cannot be smaller than $\frac{\|r\|}{2M}$. The above rank-1 construction achieves exactly that norm:
\[
   \bigl\|l\bigr\|
   \;=\;
   \frac{\|r\|}{\|2x^\star\|^2}
   \,\sup_{w\neq 0}\!
   \frac{|\langle w,2x^\star\rangle|}{\|w\|}
   \;=\;
   \frac{\|r\|}{4\,M^2} \,\|2x^\star\|
   \;=\;
   \frac{\|r\|}{2\,M}.
\]
Hence $\frac{\|r\|}{2M}$ is exactly the minimal possible value for $\|l\|$.  
\end{proof}

This lemma tells us that for a single constraint, the minimum value of weight decay corresponds to an embedding that is specific distributed on a smooth manifold of $n-1$ dimensions.

This implies $\sum_{i=0}^{p-1} \| \phi(i+1)-\phi(i) \|$ will decrease within the optimization process. In this way we can get the desired results.

Finally, we consider the general form of $f(i,j)$, which means that we need to change the constraints in the lemma to for each pair $(i,j)$ with $f(i,j) \equiv 0 \pmod{p}$, we must have
\[
   l(x_i) + l(x_j)
   \;=\;
   r,
\]

Here, the general f(i,j) cannot guarantee that the manifold has a good structure, so our theorem always holds, but the grokking phenomenon may not occur.
\end{proof}

\subsection{Proof of Theorem~\ref{thm:dist}}
\begin{proof}
    According to Theorem~\ref{thm:uniemb}, for $i \in I''$,$j \in J''$, if the model outputs the correct r(i,j), we should have:
    \[
    \|\phi(i)+\phi(j)-k_{r(i,j)}\| < \delta_{r(i,j)}.
    \]
    Using the triangle inequality we could get:
    \[
    \|\phi(i)+\phi(j)-k_{r(i,j)}\| \leq
    \|\phi(i)+\phi(j)-\phi(i')-\phi(j')\|+ 
    \|\phi(i')+\phi(j')-k_{r(i,j)}\|
    \]
    If $(i,j)\sim (i',j')$ we use Theorem~\ref{thm:med}
    \[
    \|\phi(i)+\phi(j)-k_{r(i,j)}\| \leq |i-i'+j-j'|\delta +\delta'_{r(i,j)}
    \]

    This proves case (i), the proof of case (ii) is similar.
\end{proof}

\subsection{Proof of Theorem~\ref{thm:simpledis}}
\begin{proof}

Consider a two-dimensional grid of size \( p \times p \), yielding \( N = p^2 \) distinct points (or pairs). Suppose we randomly select a subset of these points as the \textit{training set}, whose size is \( m = \alpha \times N \) for some ratio \( 0 \leq \alpha \leq 1 \). We say that a grid point \((i, j)\) is \emph{covered} by a training point \((i', j')\) if the Manhattan distance
\[
d\bigl((i, j), (i', j')\bigr) = |i - i'| + |j - j'| 
\]
is strictly less than some threshold \( k \). We define the \textit{coverage set} for a single training point \((i', j')\) as
\[
\mathcal{C}(i', j') = \Bigl\{(i, j) : \; |i - i'| + |j - j'| < k \Bigr\}.
\]
Each training point thus covers exactly \( C \) points in the grid, where 
\[
C = 1 + \sum_{d=1}^{k-1} 4d.
\]
We wish to determine how large \(\alpha\) must be for the probability of \emph{full coverage}---i.e., every point in the grid is covered by at least one training point---to be close to 1.

Let \(\alpha\) be the fraction of points selected in the training set, and let \( m = \alpha \times N \) be the size of this training set. For any fixed point \((i, j)\), the probability that this point is covered by at least one training point is
\[
P_{\mathrm{cover}} \approx 1 - \left( 1 - \frac{C}{N} \right)^m.
\]

Each training point independently covers a fraction \(\frac{C}{N}\) of all points, assuming a uniform sampling from the \(N\) grid points. Hence, the probability that a single training point \emph{does not} cover \((i, j)\) is \(1 - \frac{C}{N}\). If we select \(m\) training points independently and uniformly at random, then the probability that none of them cover \((i, j)\) is approximately 
\[
\left(1 - \frac{C}{N}\right)^m.
\]
Consequently, the probability that \((i, j)\) is covered by at least one training point becomes
\[
1 - \left(1 - \frac{C}{N}\right)^m.
\]

We further leverage the union bound to argue about the probability of \emph{full coverage} of all \( N \) points:

The probability of full coverage satisfies
\[
P_{\mathrm{full}} \;=\; \Pr\bigl(\text{all } N \text{ points are covered}\bigr)
\;\ge\;
1 \;-\; N \;\biggl(1 - \tfrac{C}{N}\biggr)^m.
\]

By the union bound,
\[
\Pr\bigl(\cup_{(i,j)} [\text{point }(i,j)\text{ not covered}]\bigr)
\;\le\;
\sum_{(i,j)} \Pr[\text{point }(i,j)\text{ not covered}]
\;=\;
N \,\left(1 - \frac{C}{N}\right)^m.
\]
Hence,
\[
P_{\mathrm{full}}
=
1 - \Pr\bigl(\cup_{(i,j)} [\text{point }(i,j)\text{ not covered}]\bigr)
\;\ge\;
1 - N \,\left(1 - \frac{C}{N}\right)^m.
\]

A simple approach is to require the \emph{expected} number of uncovered points to be less than 1, which gives
\[
N \times \bigl(1 - \frac{C}{N}\bigr)^m \;<\; 1.
\]
Taking natural logarithms on both sides and using the approximation \(\ln(1 - x) \approx -x\) when \(x\) is small, we obtain
\[
\ln(N) \;-\; m \,\frac{C}{N} \;<\; 0 
\quad\Longrightarrow\quad
m \;>\; \frac{N}{C} \,\ln(N).
\]
Since \(m = \alpha N\), this yields
\[
\alpha \;>\; \frac{1}{C}\,\ln(N) \; =\;  \frac{1}{C}\,2\ln(p).
\]

\end{proof}

 \section{Understanding Model Behavior Using Dynamical Systems}

In this appendix, we present several more complex concepts in dynamic system modeling, employing purely formal notations to link the model with system optimization. Dynamical systems give us a new perspective on classification models. We can understand the model's characteristic markers as the state vector of the dynamical system. Under this understanding, progress measure can be regarded as the analysis of the state markers, so the Poincare section corresponds to local complexity, and Fourier decomposition corresponds to signal analysis.

\begin{definition}[Dynamical system]
\label{def1}
   The following form of equations are what we refer to as a dynamical system
   \begin{equation*}
        \frac{d x}{dt}=f( x,t;\mu),
   \end{equation*}
   \begin{equation*}
         x  \rightarrow g( x;\mu),
   \end{equation*}
   with $  x \in  U  \subset  \mathbb{R}^n$,  $ t \in  \mathbb{R}^1$, and $ \mu \in  V  \subset  \mathbb{R}^p$, where $ U$ and $ V$ are open sets in $  \mathbb{R}^n$ and $  \mathbb{R}^p$. 
\end{definition}

The core of dynamical systems research lies in understanding the properties of phase space and the long-term behavior of trajectories.

\begin{definition}[Phase space and equilibrium solution]
\label{def2}   
    For a dynamical system $ \frac{d x}{dt} = f( x,t;\mu)$, some interval  $  I  \subset  \mathbb{R}^1$ into $  \mathbb{R}^n$, which we represent as $  x: ( I   \rightarrow  \mathbb{R}^n) \quad t   \rightarrow  x(t)$
    with $ \frac{d x}{dt} = f( x,t;\mu)$ satisfied. The map $  x$ called a trajectory and the space of the curve called the phase space of the dynamical system. The long-term behavior of dynamical systems is often closely related to their equilibrium points. Equilibrium points are these $  x$ make $ f( x,t;\mu)=0$.

\end{definition}

We consider a deep learning model with \( n \) layers, where each layer corresponds to a manifold mapping \(   f_n(x;\omega_n) \) with $  \omega_n$ represents the parameter set of the \(  n \)-th layer. Here, the term "layer" refers to the component that induces the manifold mapping, and does not necessarily denote a specific layer. Let the loss function be denoted as \(  L \). Following the approach in this paper, we use measures to handle the dataset. We denote the embedding of the dataset, considering each token in the embedding layer as a high-dimensional vector in \(     \mathbb{R}^n \). We denote the empirical measure corresponding to these points in \(     \mathbb{R}^n \) as \(  \nu \).

\begin{definition}
    The empirical measure  $  \mu$  in \(   \mathbb{R}^n \) for a dataset with $  N$ data points is defined as:

\[
  \mu = \frac{1}{N}  \sum_{i=1}^{N} \delta_{  x_i},
\]

where \(   x_i \in   \mathbb{R}^n \) denotes the high-dimensional vectors corresponding to each data point, and \( \delta{  x_i} \) represents the Dirac measure centered at \(   x_i \).
\end{definition}

We denote the functional relationships involved in the backward process as \(  \partial L \) and \(   \partial f_n \), the model's training process is equivalent to the following equation:
\begin{equation*}
      \frac{d\nu}{dt}=\prod_{i=1}^{n}\partial f_i\partial L,
\end{equation*}
\begin{equation*}
      \frac{d\omega_k}{dt}=\prod_{i=k+1}^{n} \partial f_i \partial L.
\end{equation*}
Now, we still lack a method for the quantitative description of a task. We assume that the training (or pre-training) task of the model corresponds to a kernel function \(  K \) and a discriminative equation.

\begin{definition}
    \(  K:\Omega \times   \mathbb{R}^n   \mathbb{R}ightarrow   \mathbb{R} \) is referred to as the kernel function corresponding to the training task if
    \begin{equation*}
          \int_{  \mathbb{R}^n} K(x,\prod_{i=1}^{n}f_i \circ d\nu)=1,
    \end{equation*}
    with $  x \in \Omega$ means the model training has succeeded.
\end{definition}

The issues corresponding to the phenomena of delayed generalization and emergence can be formulated as follows:

\begin{definition}[Delayed generalization]
    For the family of empirical measures \(   \{ \nu_i \} _{i \in \{1,...,N\}}\),  delayed generalization means there exist a positive integer \( n < N \) such that:
    \begin{equation*}
          \int_{  \mathbb{R}^n} K(x,\prod_{i=1}^{n}f_i \circ d\nu_n)=1 \Leftrightarrow   \int_{  \mathbb{R}^n} K(x,\prod_{i=1}^{n}f_i \circ d\nu_N)=1.
    \end{equation*}
\end{definition}

\begin{definition}[Emergence]
    For the family of kernel functions \(   \{ K_i \} _{i \in \{1,...,N\}}\),  emergence means there exist a positive integer \( n < N \) such that:
    \begin{equation*}
           \int_{  \mathbb{R}^n} K_n(x,\prod_{i=1}^{n}f_i \circ d\nu)=1 \Leftrightarrow   \int_{  \mathbb{R}^n} K_i(x,\prod_{i=1}^{n}f_i \circ d\nu)=1,\forall i \in \{1...N\}.
    \end{equation*}
\end{definition}

Specifically for the task of arithmetic over prime numbers, we can design the signature to fall into the cone $V_i$ and its original image. Thus, the parameters updated by each sample point have overlaps. Therefore, if we assume that a time step corresponds to an epoch of parameter updates, it becomes challenging to formulate a complete function \(  f (x,t;\mu)\), so we establish the dynamical system model corresponding to each individual sample for the task.
\begin{equation}
\label{ode1}
    \frac{dW_{U}}{dt} = \sum_{i=0}^{p-1} \gamma_ie^{(i)}(Wx)^{T},
\end{equation}
\begin{equation}
    \label{ode2}
\frac{dW}{dt} = \sum_{i=0}^{d}(\gamma W_{U})_{i,:}e^{(i)}x^T,
\end{equation}
\begin{equation}
    \label{ode3}
    \frac{dx}{dt}=\gamma W_{U}W.
\end{equation}

with $e_i$ represents the $i_{th}$ standard unit vector, $x$ represents the vector obtained by concatenating the three embedding vectors corresponding to $  (i,j,p_+)$. This model can be generalized to general classification tasks. Additionally, let us denote the reachable region $  \Omega$ of the model’s updating process as the state space of the dynamical system. $  \Omega$ is a finite open set.

One of the things we've been emphasizing is that each sample of a test set determines a corresponding dynamical system. In this task, due to the simplicity of the model architecture, these dynamical systems have similar phase space structures (only one affine transformation from each other). Using this similarity, we can define a family of transformations that project a higher-dimensional phase space onto some lower-dimensional phase spaces ($ \mathbb{R}^n \rightarrow \mathbb{R}^m $) to simplify our study of the properties of the system.

For the \eqref{ode1}-\ref{ode3}, we define a mapping family $  \{ F_{c,i,j,p}\}$  dependent on the entire sample set.

\begin{definition}
\label{def3}
    For the sample $  (i,j,p_+)\rightarrow c$, we define a map $  F_{s,i,j,p},s=1,2,...,p$ that it maps $x,W,W_U$ in \eqref{ode1}-\ref{ode3} into three scalars $x,w,u$ with equations 
\begin{equation}
\label{ds1}
    \frac{dx}{dt}=-(xwu-\delta_{s,c})wu, 
\end{equation}
\begin{equation}
      \label{ds2}
    \frac{dw}{dt}=-(xwu-\delta_{s,c})xu, 
\end{equation}
\begin{equation}
          \label{ds3}
    \frac{du}{dt}=-(xwu-\delta_{s,c})xw.
  \end{equation}

The $  \delta_{s,c}$ mentioned above is a Kronecker symbol.
\end{definition}

\begin{theorem}
\label{th1}
    The map family defined above exists.
\end{theorem}
The key idea of this theorem lies in quantifying the changes in task characteristics with model updates into a three-dimensional dynamical system. For the simplified dynamical system, we consider the properties of its trajectories within the unit cube \([0,1]^3\).

\begin{theorem}
\label{th2}
    The non-trivial attractors (equilibrium points) of the dynamical system defined by \eqref{ds1}-\ref{ds3} are all stable.
\end{theorem}

\begin{proof}
Theorem~\ref{th1} is an existence theorem, which is used to ensure the existence of the framework for our method.

We employ Urysohn’s Lemma from topology to establish the existence result here. Urysohn’s Lemma is a classical tool in topology, often used to demonstrate the construction of specific types of continuous functions within certain topological spaces.
    \begin{lemma}[Urysohn's lemma]
 \label{le1}
        Let $X$ be a normal topological space, and let $A$ and $B$ be two disjoint closed subsets of $X$. Then, there exists a continuous function $ f: X \rightarrow [0,1]$ such that $f(A)=\{0\}$ and $f(B)=\{1\}$
    \end{lemma}

Remember that $V_c=\{x|x_c > x_i, \forall  i \neq c,1 \leq i \leq p\}$, $W_uWx$ is a vector in $\mathbb{R}^{p}$. After the model training is completed, we consider the convex hull of all vectors that lie within the cone \(V_c\), denoted as \(B_c\). And the complement of $V_c$  in $ \bar{\Omega}$ denoted as $A_c$.

Since $  B_c \subset\subset V_c$, \(  A_c\) and \(  B_c\) are two separable closed sets. Moreover, since every Euclidean space is a normal space, we can invoke Urysohn's lemma. Thus, there exists a continuous function \(  f : \mathbb{R}^p \rightarrow [0,1]\) such that \(   f(A_c) = \{0\} ,f(B_c)=\{1\}\). For a vector \(   W_uWx \) in $  \mathbb{R}^p$, we denote \(   f(W_uWx) = uwx \).

Now, we will employ a similar approach to derive the complete mapping that we require. Consider the vector space $  \mathbb{R}^n$ and closed sets $  A_{c}'$,$  B_{c}'$ such that $  W_uA_{c}'=A_c$, $W_uB_{c}'=B_c$. Using Urysohn's lemma like above we obtain there exists a continuous function \(  f' : \mathbb{R}^n \rightarrow [0,1]\) such that \(   f'(A_c') = \{0\} ,f'(B_c')=\{1\}\). For a vector \(   Wx \) in $  \mathbb{R}^n$, we denote \(   f'(W_uWx) = wx \). Similarly, we can obtain that the function $  f'': \mathbb{R}^n \rightarrow [0,1]$ and denote $  f''(x)=x$.

According to our definition, when the model makes a correct prediction, \( x = w = u = 1 \); when the model makes an incorrect prediction, \( x = 0 \). This satisfies the \eqref{ds1}-\ref{ds3} when \( s = c \). When \( s \neq c \), we can use a similar approach by extracting a closed set from the high-dimensional cone \(  V_s \) and applying the Urysohn lemma to it along with the remaining set. It should be noted that this time the complement set is mapped to 1, while the subset is mapped to 0.

Theorem~\ref{th2} aims to demonstrate that the compressed dynamical system within the unit cube exhibits simple trajectory properties, thereby indicating that the complex, high-dimensional nature of the dynamical system is not a phenomenon that requires specific consideration for this task.

We take the following system of equations as an example, with similar reasoning applying to other cases.
\begin{align}
      \frac{dx}{dt}=(1-xwu)wu, \\
    \frac{dw}{dt}=(1-xwu)xu,  \\
    \frac{du}{dt}=(1-xwu)xw.
\end{align}
We first provide the precise definition and the criteria for determining the stability of equilibrium points in dynamical systems.

\begin{definition}
Consider an equilibrium point \(   x^* \) of a dynamical system.

\begin{itemize}
    \item \(  x^* \) is said to be stable if, for any \(  \epsilon > 0\), there exists a \(  \delta > 0\) such that if \(   \|x(0) - x^*\| < \delta \), then \(   \|x(t) - x^*\| < \epsilon \) for all \(   t \geq 0 \).
    
    \item \(  x^* \) is said to be asymptotically stable if it is stable and there exists a \(  \delta' > 0\) such that if \(  \|x(0) - x^*\| < \delta' \), then \(   \lim_{t \to \infty} x(t) = x^* \).

    \item \( x^* \) is said to be unstable if it is not stable; that is, there exists an \(  \epsilon > 0\) such that for any \(  \delta > 0\), there exists an initial condition \(   x(0) \) with \(   \|x(0) - x^*\| < \delta \) but \(   \|x(t) - x^*\| \geq \epsilon \) for some \(   t \geq 0 \).
\end{itemize}

\end{definition}

To determine the stability category of an equilibrium point, the Lyapunov method is often used.

\begin{theorem}[ Lyapunov's stability theorem]
  Consider a dynamical system described by:  \[
\frac{dx}{dt} = f(x), \quad x \in \mathbb{R}^n,
\]

where \(x^* \) is an equilibrium point, i.e., \( f(x^*) = 0 \). To determine the stability of the equilibrium point \( x^* \), one constructs a Lyapunov function \( V: \mathbb{R}^n \to \mathbb{R}\) that satisfies the following conditions:
\begin{enumerate}
    \item \( V(x) > 0 \) for all \( x \neq x^* \), and \( V(x^*) = 0 \) (positive definiteness).
    \item The time derivative of \( V(x) \) along the trajectories of the system, given by \( \dot{V}(x) = \nabla V \cdot f(x) \), satisfies \(  \dot{V}(x) \leq 0 \) (negative semi-definiteness).
\end{enumerate}

Then:
\begin{itemize}
    \item If there exists a Lyapunov function such that \(  \dot{V}(x) < 0 \) for all \( x \neq x^* \), then the equilibrium point \(x^* \) is locally asymptotically stable.
    \item If \( \dot{V}(x) \leq 0 \), then the equilibrium point is stable in the sense of Lyapunov, but not necessarily asymptotically stable.
    \item If no such Lyapunov function can be found, alternative methods must be used to analyze stability.
\end{itemize}
\end{theorem}

For our equations we construct the function $  V(x,w,u)$ as follows:
\begin{equation}
      V(x,w,u)=(x-1)^2+(w-1)^2+(u-1)^2.
\end{equation}
It is obvious that $  V>0$ for all points in open unit cube and $  V=0$ when the system takes (1,1,1). Now we only need to verify that \(   \dot{V} \) is semi-positive definite within the unit cube.
\begin{equation}
    \dot{V}= (3xwu-wu-xu-xw)(1-xwu)
\end{equation}
It is easy to verify that all of its leading principal minors are non-positive, hence it is non-positive definite. So (1,1,1) is a stable equilibrium point.

We can intuitively observe the simplicity of the structure of this dynamical system through visualization. We present the visualization results in Figure~\ref{vision}.

\begin{figure}[ht]
    \centering
    \setlength{\abovecaptionskip}{0.cm}
    \includegraphics[width = \textwidth ]{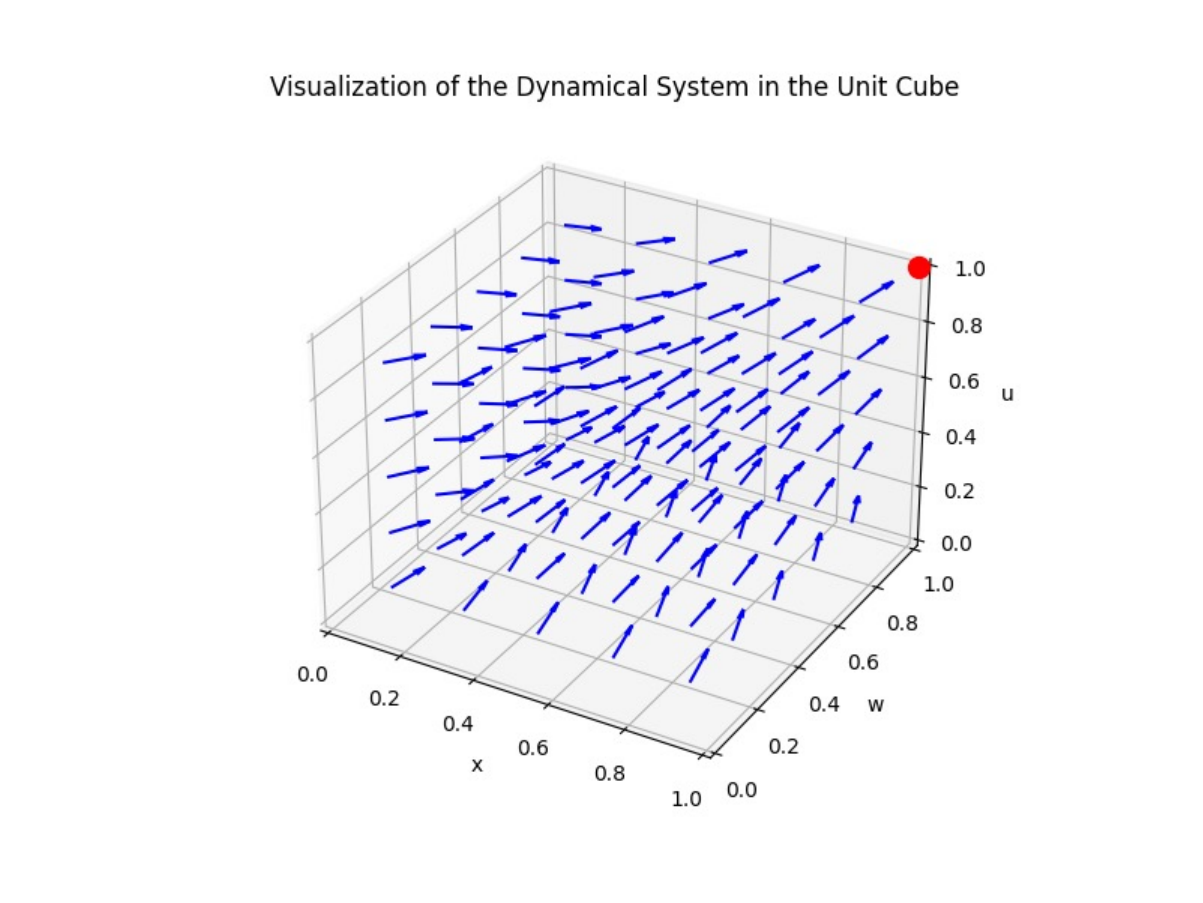}

    \caption{
    This figure directly show the phenomenon of points within the unit cube moving towards the equilibrium point (1,1,1) which is highlighted.
    }
    \label{vision}
\end{figure}

\end{proof}

The above theorem guarantees the existence of grokking.

\section{Supplementary Experiments}
\label{support}
\subsection{Extra Discussion on Upper Limit of Accuracy }
Throughout the main text, our discussion of accuracy improvement has assumed that the model's embedding dimension is sufficiently large to capture all necessary information. Naturally, reducing this dimensionality can severely compromise the model's stability, particularly when the proportion of randomly selected training sets approaches the critical threshold described in Theorem~\ref{thm:simpledis}. We therefore explore the consequences of this instability. Note that the experimental results here are highly dependent on the randomly selected training set structure and may be difficult to reproduce. In what follows, \(d_{model}\) denotes the model’s embedding dimension, i.e., the token vector length and we take $I=\{0,1,...,p-1\}$,$J=\{k,k+1,...,k+p-1\}$.

First, we set \(p = 97\) and \(k = 85\). Under these conditions, grokking occurred more than 20{,}000 epochs earlier with \(d_{model} = 256\) compared to \(d_{model} = 128\). Next, we set \(p = 47\) and \(k = 0\) and observed that grokking took place over 10{,}000 epochs sooner with \(d_{model} = 256\) than with \(d_{model} = 64\). These findings are illustrated in Figure~\ref{se1}.

\begin{figure}[ht]
    \centering
    \setlength{\abovecaptionskip}{0.cm}
    \includegraphics[width = \textwidth ]{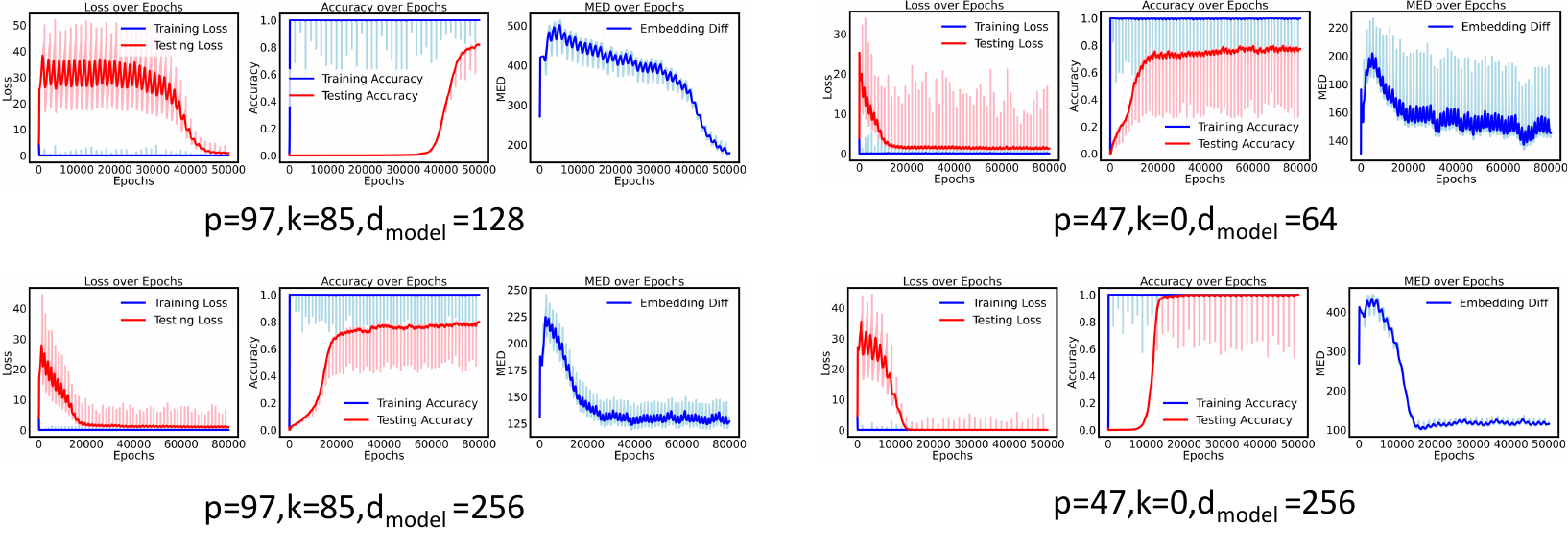}
    \caption{(Left) We selected $ p=97$ and $  k=85$ and we found that grokking occurred over 20,000 epochs earlier with \(  d_{model} = 256\) compared to \(  d_{model} = 128\). (Right) We selected $  p=47$ and $  k=0$ and we found that grokking occurred over 10,000 epochs earlier with $  d_{model} = 256$ compared to \(  d_{model} = 64\). 
    }
    \label{se1}
\end{figure}

Next, we investigated the scenario \(p = 43\) and \(k = 0\). In contrast to the previous experiments, it is evident that increasing the model's embedding dimension raises the final upper bound of test accuracy. Specifically, when the embedding dimension is set to 128, the test accuracy fluctuates around 0.7, but when the dimension is increased to 512, the test accuracy eventually reaches 0.99. These results are shown in Figure~\ref{se2}.

\begin{figure}[t!]
    \centering
    \setlength{\abovecaptionskip}{0.cm}
    \includegraphics[width = \textwidth ]{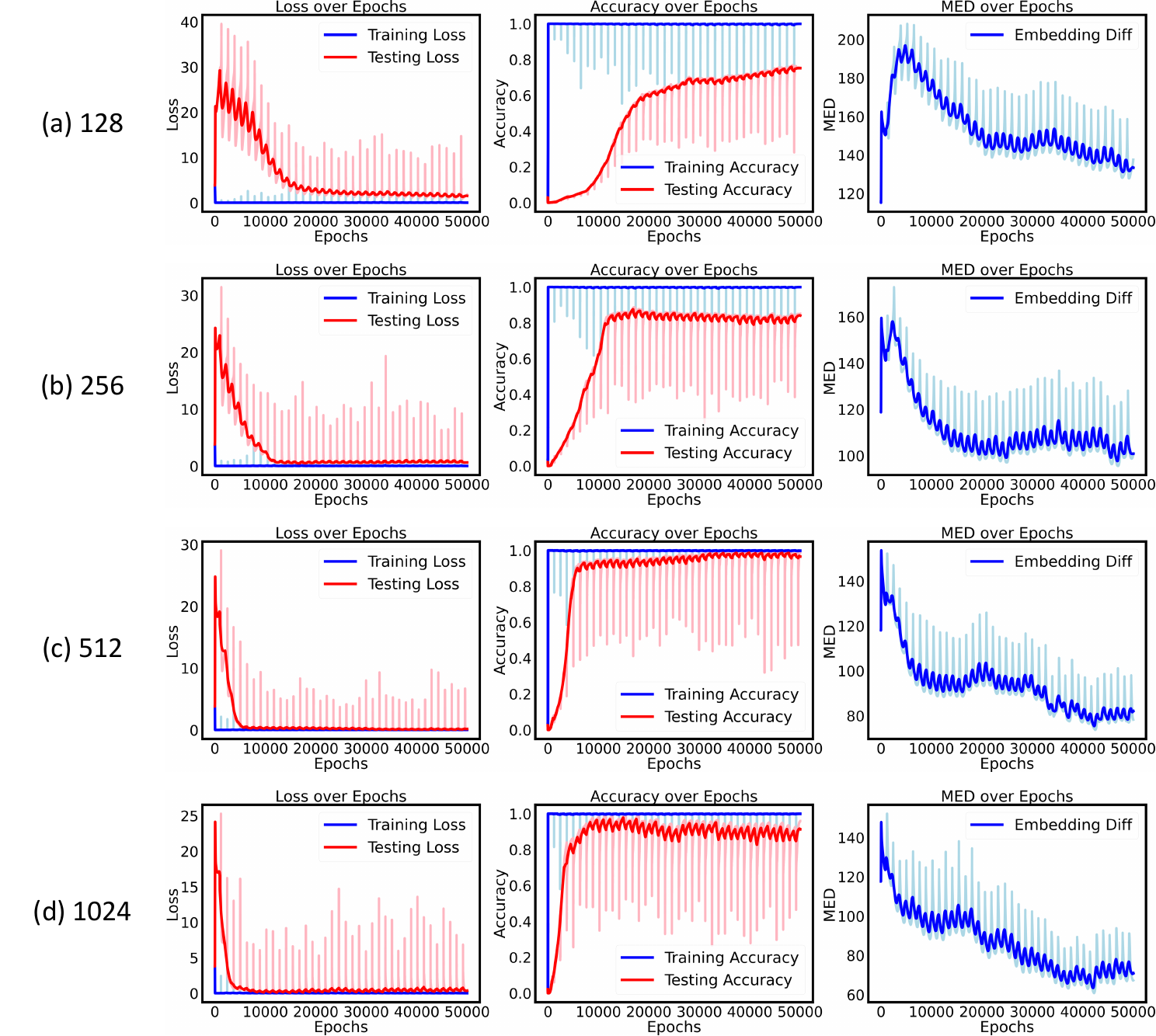}
   
    \caption{
    This figure aims to compare the impact of different embedding dimensions on the speed and upper limit of generalization. We selected the same number of epochs for all cases. We found that although increasing the embedding dimension from 512 to 1024 resulted in a decrease in the lower bound of the MED from around 80 to around 60, there was no significant improvement in generalization ability.
    }
    \label{se2}
\end{figure}

However, the results mentioned above do not represent boundary cases. To observe behavior at the boundary, we continuously increased the value of \(k\) until grokking ceased to appear at lower embedding dimensions. Then, we further increased the embedding dimension. The outcome is striking: enlarging the embedding dimension led grokking to re-emerge. These findings are presented in Figure~\ref{se3}.

This result is intuitive: a larger embedding dimension expands the recognition region of the cone, and there are more options for homogenization on a co-dimension-1 manifold. Consequently, the grokking phenomenon becomes more likely to occur.
\begin{figure}[t!]
    \centering
    \setlength{\abovecaptionskip}{0.cm}
    \includegraphics[width = \textwidth ]{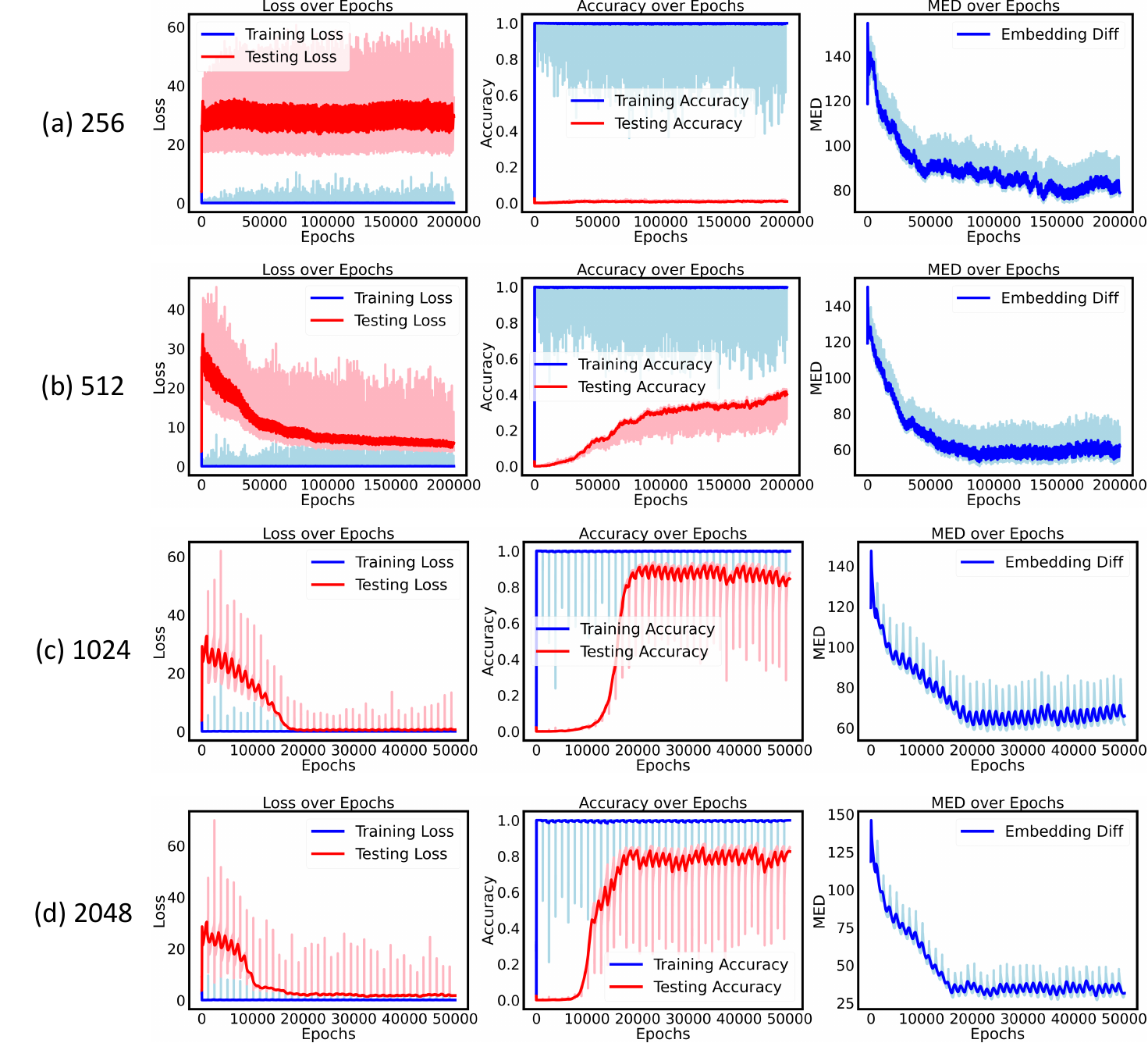}
    
    \caption{
    We discarded the case with an embedding dimension of 128 because it was identical to the plot for the 256-dimensional case. Note that we extended the number of epochs for the 256 and 512 dimensions to 200,000 to demonstrate that they had already reached the upper limit of their generalization abilities. When the embedding dimension was increased from 1024 to 2048, we observed a similar outcome to the previous experiment, with no noticeable change.
    }
    \label{se3}
\end{figure}

\subsection{Other Operations over Prime Fields}

We have discussed the impact of choosing different $f(i,j)$ on the upper bound of accuracy, and have also described that this does not affect the ability of the MED function to monitor the progress of parameter updates due to weight decay, and our subsequent experiments report this result. We selected five cases to represent this: \(  x - y \), $  xy$, \(  x^2 + y^2\), \(  x^2 + xy + y^2\), \(  x^3 + y^3\), and \(  x^3 + xy + y^3\). Through the experiments, we found that our MED function is still exceptionally capable of tracking test loss.  For detailed experiments, please refer to Figure\ref{se4}-\ref{se11}
\begin{figure}[!ht]
    \centering
    \setlength{\abovecaptionskip}{0.cm}
    \includegraphics[width = \textwidth ]{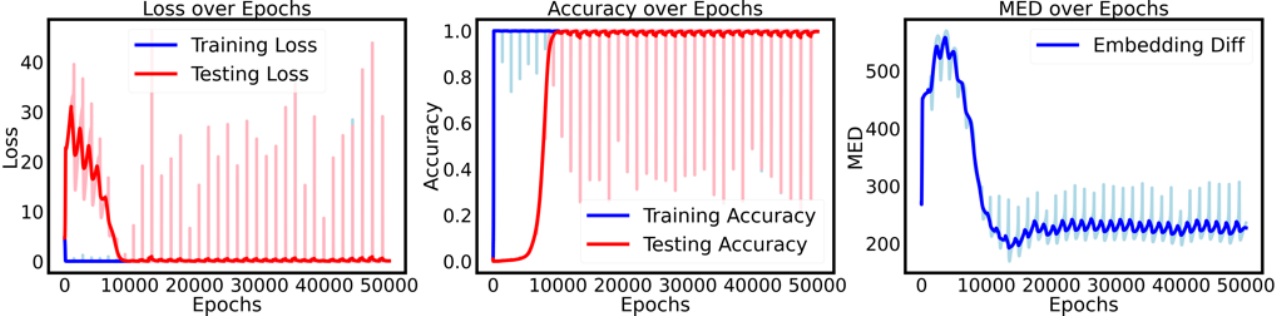}
    
    \caption{
    $  p=97$, $  k=0$, $  d_{model}$=128, $  x-y$, the training set proportion is set to 0.3.
    }
    \label{se4}
\end{figure}
\begin{figure}[!ht]
    \centering
    \setlength{\abovecaptionskip}{0.cm}
    \includegraphics[width = \textwidth ]{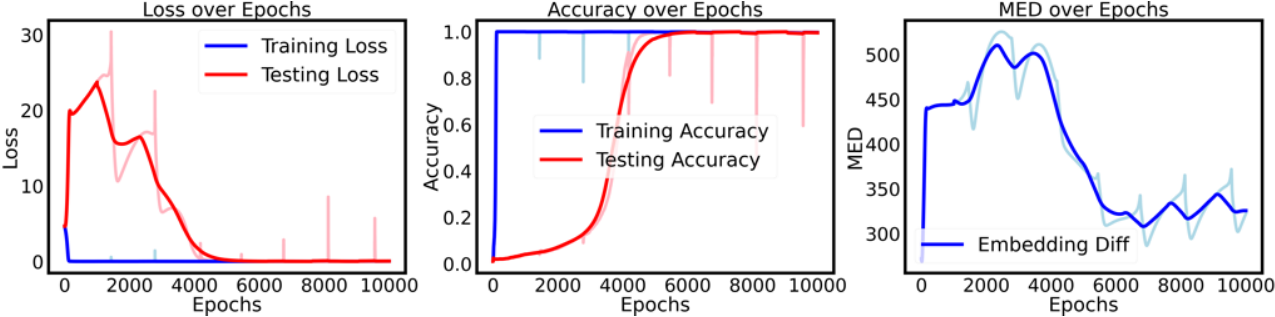}
 
    \caption{
    $  p=97$, $  k=0$, $  d_{model}$=128, $  xy$, the training set proportion is set to 0.3.
    }
    \label{se5}
\end{figure}
\begin{figure}[!ht]
    \centering
    \setlength{\abovecaptionskip}{0.cm}
    \includegraphics[width = \textwidth ]{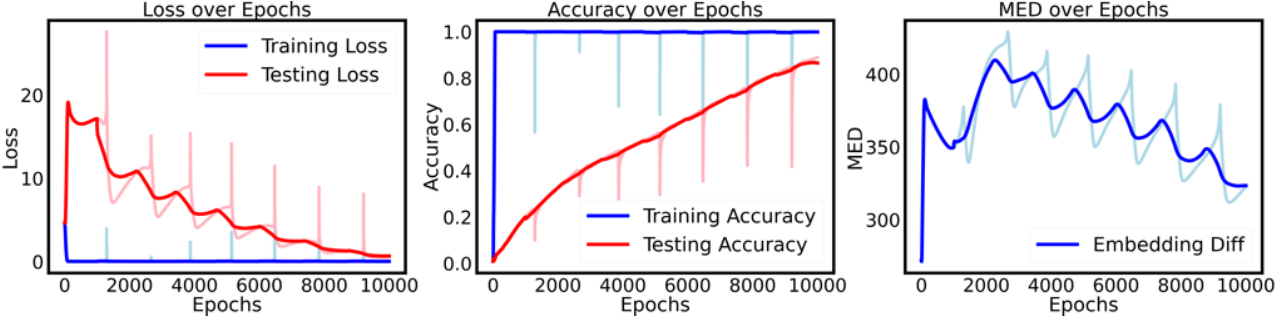}
      \caption{
    $  p=97$, $  k=0$, $  d_{model}$=256, $  x^2+y^2$, the training set proportion is set to 0.2.
    }
    \label{se6}
\end{figure}
\begin{figure}[!ht]
    \centering
    \setlength{\abovecaptionskip}{0.cm}
    \includegraphics[width = \textwidth ]{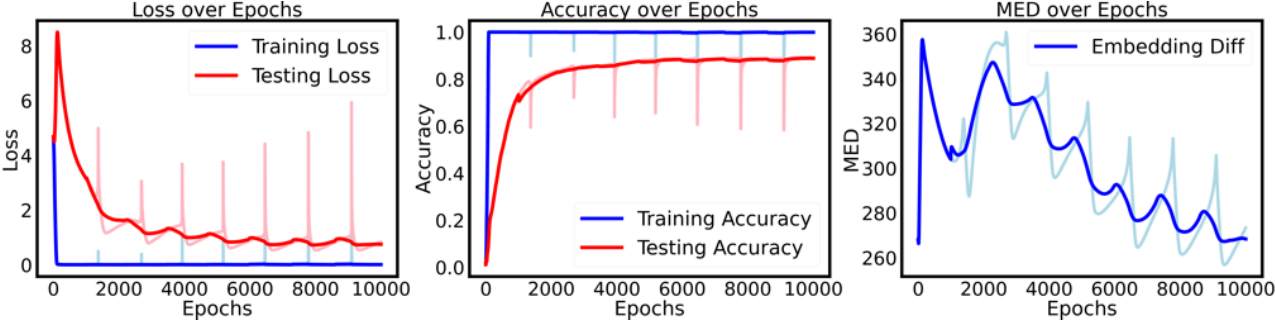}
    
    \caption{
    $  p=97$, $  k=0$, $  d_{model}$=128, $  x^3+y^3$, the training set proportion is set to 0.2.
    }
    \label{se7}
\end{figure}
\begin{figure}[!ht]
    \centering
    \setlength{\abovecaptionskip}{0.cm}
    \includegraphics[width = \textwidth ]{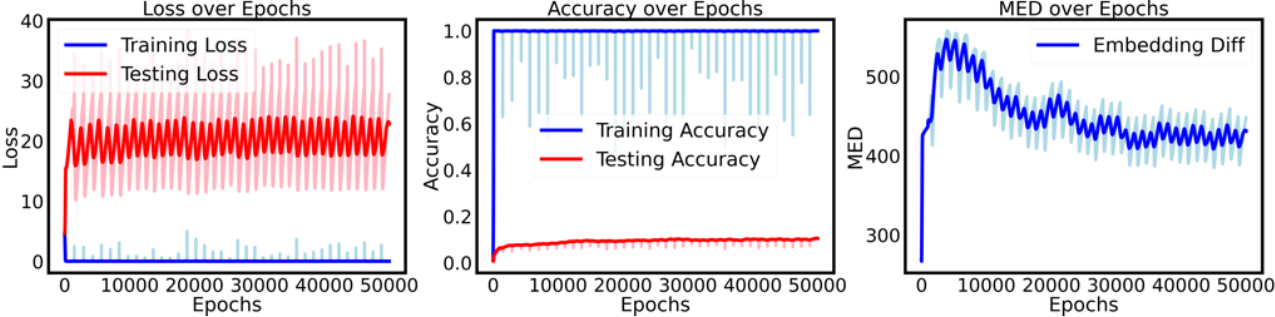}
    
    \caption{
    $  p=97$, $  k=0$, $  d_{model}$=128, $  x^2+xy+y^2$, the training set proportion is set to 0.3.
    }
    \label{se8}
\end{figure}
\begin{figure}[!ht]
    \centering
    \setlength{\abovecaptionskip}{0.cm}
    \includegraphics[width = \textwidth ]{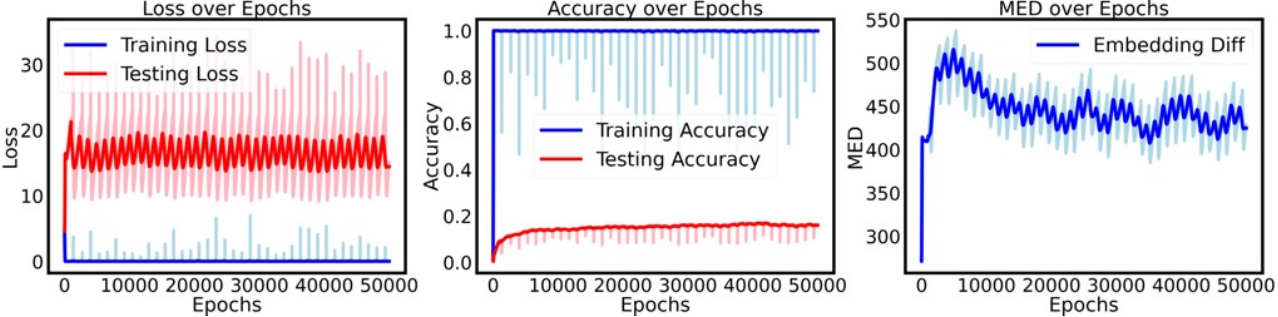}

    \caption{
    $  p=97$, $  k=0$, $  d_{model}$=256, $  x^2+xy+y^2$, the training set proportion is set to 0.3.
    }
    \label{se9}
\end{figure}
\begin{figure}[!ht]
    \centering
    \setlength{\abovecaptionskip}{0.cm}
    \includegraphics[width = \textwidth ]{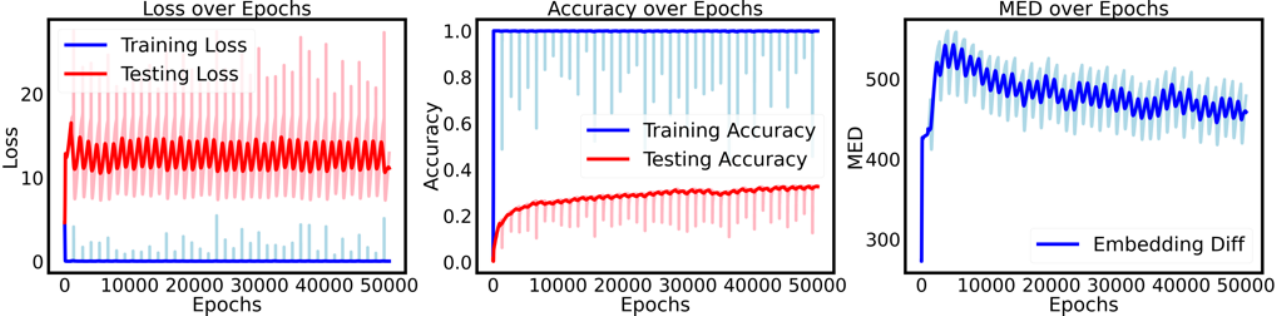}
  
    \caption{
    $  p=97$, $  k=0$, $  d_{model}$=256, $  x^2+xy+y^2$, the training set proportion is set to 0.4.
    }
    \label{se10}
\end{figure}
\begin{figure}[!ht]
    \centering
    \setlength{\abovecaptionskip}{0.cm}
    \includegraphics[width = \textwidth ]{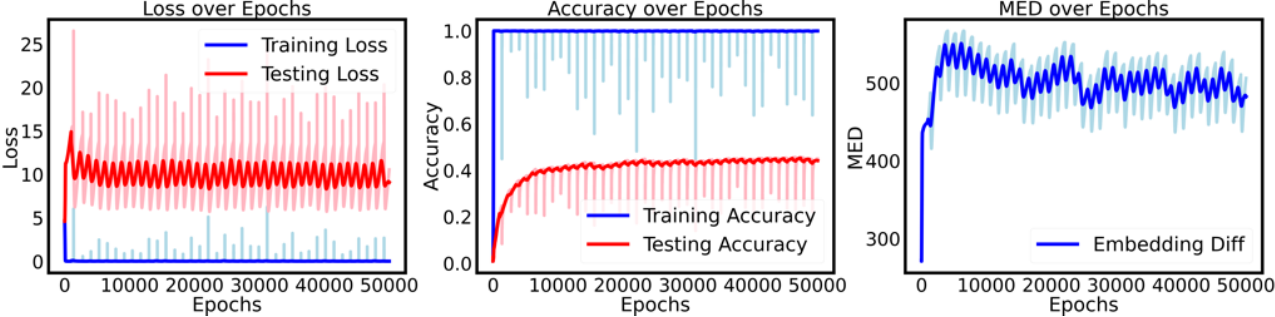}
    \caption{
    $  p=97$, $  k=0$, $  d_{model}$=256, $  x^3+xy+y^3$, the training set proportion is set to 0.5.
    }
    \label{se11}
\end{figure}

\FloatBarrier
\newpage
\section{Can Special Initialization Speed up Grokking?}
After thoroughly understanding the causes and monitoring methods of grokking, we will talk on a specific question: can special initialization speed up grokking? Structured initialization representation can be regarded as a special regularization technique, which is a consensus in representation learning.

Since the occurrence of grokking implies thorough learning of the uniformity in the embedding space, our approach is to adjust the initial embedding values to facilitate this learning process. Specifically, we use a circulant matrix generated from a random vector to replace the embedding matrix, as expressed below:

\begin{equation}
    W_E = \text{Toeplitz} (v,v'),
\end{equation}
with $\text{Toeplitz}$ represents the operation of generating a Toeplitz matrix, and \( v \) is a random vector. We have presented the experimental results in Figure~\ref{acceleration}

\begin{figure}[ht]
    \centering
    \setlength{\abovecaptionskip}{0.cm}
    \includegraphics[width = \textwidth ]{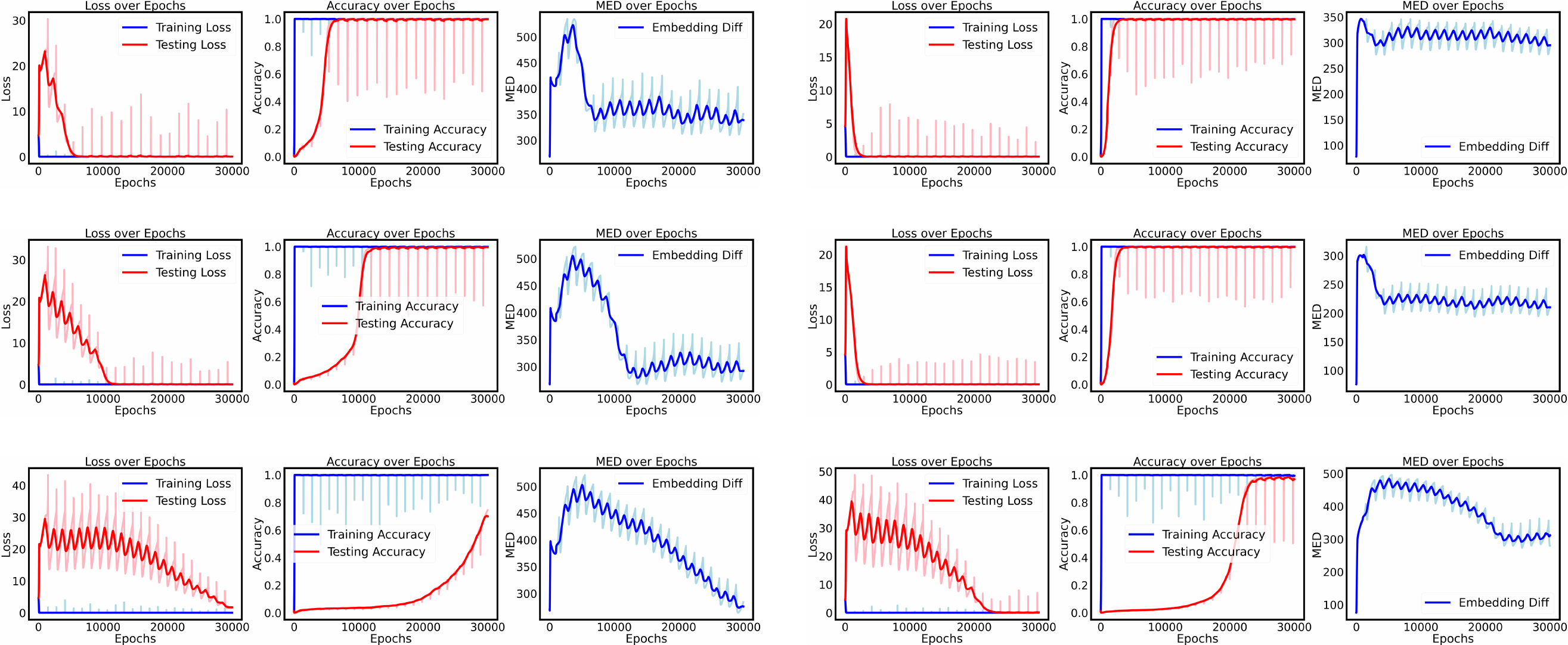}.
    \caption{
    The experiments were conducted with \(   p = 97 \), \(   k = 0 \), and \(   d_{model} = 128 \). The left side shows the results with random embeddings, while the right side displays the results using the improved embedding method. From top to bottom, the training set proportions are 0.3, 0.25, and 0.2, respectively. The number of epochs required for grokking to begin on the right side is only half of that on the left side.    }
    \label{acceleration}
\end{figure}

\end{document}